\documentclass[11pt]{article}
\usepackage[letterpaper, left=1in, right=1in, top=1in,bottom=1in]{geometry}
\usepackage{amsfonts}       
\usepackage{nicefrac}       
\usepackage{bbm}
\usepackage{bm}
\usepackage{tikz}
\usetikzlibrary{positioning,chains,fit,shapes,calc,arrows}
\usetikzlibrary{arrows.meta}
\usepackage{graphicx}
\usepackage{pgf}
\usepackage{mathtools}
\usepackage{mathrsfs}
\usepackage{subfigure}
\usepackage{amsmath,amsthm,amssymb}
\usepackage{comment}
\usepackage{xfrac}
\usepackage{accents}
\usepackage{cancel}
\usepackage[normalem]{ulem}

\usepackage{algorithm}
\usepackage{algpseudocode}

\usepackage[round]{natbib}

\usepackage[pagebackref=true]{hyperref} \PassOptionsToPackage{pagebackref=true}{hyperref}

\DeclareMathAlphabet{\pazocal}{OMS}{zplm}{m}{n}

\usepackage{auxiliary}

\usepackage{float}

\usepackage{url}      

\usepackage[pagebackref=true]{hyperref} \PassOptionsToPackage{pagebackref=true}{hyperref}   
\usepackage[capitalize]{cleveref}

\NewDocumentEnvironment{myproof}{o}
  {\IfNoValueTF{#1}{\paragraph{{{\small Proof.}}}} {\paragraph{{{\small #1.}} }} }
  {\hfill$\qed$}

\usepackage{color}              
\usepackage[suppress]{color-edits}
\addauthor{grace}{blue} 

\date{}

\title{
A Hierarchy of Policy Learning Problems
}
\author{
Hamsa Bastani\thanks{Wharton School, University of Pennsylvania, \texttt{hamsab@wharton.upenn.edu}}
\and 
Osbert Bastani\thanks{University of Pennsylvania, \texttt{obastani@seas.upenn.edu}}
\and 
Shihan Chen\thanks{Graduate Group in Applied Mathematics and Computational Science, University of
Pennsylvania, \texttt{gracsh@sas.upenn.edu}}
}

\begin{document}

\maketitle

\thispagestyle{empty}

\begin{abstract}
Policy learning has received substantial recent attention with the goal of learning policies from observational data for either automated or human-in-the-loop decision-making. A majority of work in this space has focused on developing algorithms for computing policies that minimize regret compared to the optimal policy. However, in many practical settings, there is simply insufficient data to obtain low regret. As a consequence, recent work has shifted attention to alternative objectives---most notably, studying whether it is possible to learn an \emph{improving policy} that  statistically significantly outperforms baseline policies (e.g., choose constant or random actions). We argue that there is substantial merit in studying a still broader range of policy learning problems---when there is insufficient data to even learn an improving policy, there may still be useful questions that can be answered based on the available data. To this end, we provide a mathematical framework for studying the relationships between different policy learning problems. We formalize three policy learning problems within our framework: beyond the optimal policy problem and the improved policy problem, we also propose the policy existence problem, which aims to determine whether an improving policy exists. Within our framework, we can straightforwardly show that the policy existence problem reduces to the improving policy problem, which in turn reduces to the optimal policy problem; these reductions prove that each problem is at least as easy as the next one (in terms of sample complexity). A key question remains: is this hardness strict? We provide partial answers to this question. First, the gap between the optimal policy and improving policy problems is straightforwardly strict. For the improving policy and policy existence problems, we prove that a sublinear polynomial gap exists under certain natural conditions on algorithms for the improving policy learning problem. Thus, we may be able to answer questions about the existence of an improving policy even when we cannot find one. These results highlight the value in studying a broader range of policy learning problems.
\end{abstract}

\newpage
\setcounter{page}{1}

\section{Introduction}
\label{sec:intro}

At its core, the goal of data-driven decision-making is to learn a \emph{policy} for choosing actions based on relevant features of the task at hand---e.g., to target text messages for encouraging patients to get vaccinated~\citep{milkman2022680,shchetkina2024heterogeneity}, or to target COVID-19 tests to identify the largest possible number of infected travelers~\citep{bastani2021efficient}. We focus on the offline setting, where we are given an observational dataset and our goal is to estimate a policy based on this dataset; for simplicity, we assume that the observational policy is known. Much of the literature has focused on learning the \emph{optimal policy} that maximizes a given objective function~\citep{dudik2011doublyrobustpolicyevaluation}. These algorithms typically come with regret guarantees, showing that with sufficiently large training data, the estimated policy converges to an optimal one.

However, in many practical scenarios, the observational data is too limited to accurately estimate the optimal policy. Then, the decision-maker's goal is often more modest---e.g., to obtain an \emph{improving policy} that outperforms over reasonable baselines. Typically, they additionally want these improvements to be statistically significant. Recent work has proposed algorithms for learning improving policies~\citep{banerjee2025selecting,chernozhukov2025policy,bastani2025beatingwinnerscurseinferenceaware}.

This discussion motivates a broader question---given a limited observational dataset, what is the space of problems that can be reliably solved? We argue that the problems of learning optimal policies and learning improving policies are just two examples of a broader space of problems that may be of interest to a decision-maker. Beyond even computing a usable policy, the decision-maker may want to answer questions about whether an improving policy exists, or whether the treatment effects are heterogeneous. Answers to these questions can be used to guide decisions about whether to collect a larger observational dataset, collect additional features from existing individuals, run a randomized controlled trial to gather higher-quality data, or to simply forgo data-driven decision-making.

In this paper, we propose a mathematical framework for reasoning about these questions. Our framework uses a very simple model of the decision-making problem; instead, our goal is to answer questions about \emph{relationships} between different algorithms across different policy learning problems. At a high level, our framework is built around the statistical notion of characterizing the false negative rate (FNR) (i.e., type II error rate) under a constraint on the false positive rate (FPR) (i.e., type I error rate). Then, the sample complexity of an algorithm is the number of samples at which it achieves a desired FNR under a given FPR constraint. This formulation makes it easy to provide a unified definition of correctness for algorithms with different output spaces, ranging from policies (for the optimal policy problem) to policies with abstention (for the improving policy problem) to binary values (for the policy existence problem). This unified notion of correctness makes it easy to define reductions between problems, enabling us to formalize simple notions such as the idea that the optimal policy problem is at least as hard as the improving policy problem, which is in turn at least as hard as the policy existence problem.

A key remaining question is establishing gaps in the opposite direction---i.e., are there any instances where one problem is strictly harder than another? It is straightforward to see that learning an optimal policy is strictly harder than learning an improving policy---if the true treatment effects for some individuals are very small, then a large number of samples are required to determine the optimal treatment for those individuals, yet an improving policy does not need to correctly treat these individuals since the decision in inconsequential.

However, the gap between the improving policy problem and the policy existence problem is much more nuanced. Our main technical contribution is a partial answer to this question. We show that there if we restrict to algorithms for the improving policy problem that satisfy a natural monotonicity condition, then there is a substantial gap between the two problems. More broadly, our work raises many interesting questions about the policy learning for future work, including both providing sharper analyses of the three policy learning problems we study, as well as the study of additional policy learning problems such as identifying heterogeneous effects.

\subsection{Related Work}

There has been substantial recent interest in \textit{offline policy learning}, where data-driven decision-making policies must be learned from observational data collected under a fixed behavioral policy. A major line of research in this area focuses on developing reliable policy evaluation methods and using them to improve policy learning algorithms. 
This work builds on the foundational work of \citet{rosenbaum1983central}, which introduce the standard potential outcomes framework for evaluating treatment effect estimation. \citet{dudik2011doublyrobustpolicyevaluation} categorizes policy evaluation strategies into the \textit{direct method}, which  evaluates candidate policies using model-based predictions of counterfactual rewards, and the \emph{in-direct method}, which uses inverse propensity weighting (IPW). In practice, IPW is an effective evaluation strategy since it is unbiased as long as the propensity weights for the observational dataset are known, and a long line of algorithms have been developed based on these techniques~\cite{qian2011performance,zhang2012estimating,zhao2014doubly,swaminathan2015batch,zhou2023offline, zhan2023policy}. Recent works among these include \citet{kitagawa2018treated}, who proposed an algorithm with matching upper and lower regret bounds of $O(1/\sqrt{n})$ and $\Omega(1/\sqrt{n})$ respectively (where $n$ is the sample size), and \citet{athey2020policylearningobservationaldata}, who extended these results to settings where the propensities must be estimated from data (they make the unconfoundedness assumption, so propensity weights can be accurately estimated given a sufficiently large training dataset). These approaches all focus on the optimal policy problem, where the goal is to learn a policy that minimizes regret compared to the optimal policy. In contrast, we study relationships between different policy learning problems. In terms of evaluation methodology, we adopt the direct method---our model is well-specified, so we can obtain confidence intervals for estimates of counterfactual outcomes that are valid at finite sample. An interesting direction would be extending our approach to IPW evaluation.


Recently, several studies have shifted attention away from strict optimality and toward weaker, more robust objectives. Most notably, instead of asking whether a learned policy is optimal, they study whether a policy can be shown to improve upon a baseline policy with high confidence. \citet{banerjee2025selecting} proposes an algorithm that ``clusters'' policies in the policy class to improve the likelihood of obtaining statistically significant policy improvement. More recently, \citet{bastani2025beatingwinnerscurseinferenceaware,chernozhukov2025policy} both propose policy learning algorithms that aim to obtain statistically significant policy improvement by characterizing the Pareto frontier of regret and probability of outperforming the baseline (in their case, the observational policy). Finally, \cite{shchetkina2024heterogeneity} study when treatment-effect heterogeneity is \emph{actionable}, which is essentially our policy existence problem; however,
they focus on identifying conditions under which personalization can outperform the best uniform policy, and do not provide an algorithm for actually solving this problem. These approaches highlight that under observational data and finite samples, learning an optimal policy can be fundamentally harder than identifying a provably better one. As a consequence, a decision-maker may be able to establish that there exists an improving policy even when they cannot identify a specific improving policy. This result illustrates how thinking about a broader range of policy learning problems can have tangible benefits by providing useful answers when obtaining an optimal or improving policy is infeasible.

Although recent work has begun to address problems beyond optimal policy learning, there lacks a unified framework for systematically comparing different policy learning problems. Our work aims to fill this gap by providing a mathematical foundation for establishing reductions between different policy learning problems.

\subsection{Contributions}

Our contributions are two-fold. First, we provide a mathematical formalization of policy learning problems that enables us to rigorously analyze sample complexity gaps between these problems. The key challenge is that the different policy learning problems have different output spaces, making it hard to directly compare one problem to another. We formalize correctness in a way that circumvents these problems. At a high level, all problems take as input an observational dataset $Z$ sampled according to an underlying (unknown) \emph{instance} $P\in\mathcal{P}$ representing the observational data distribution. Then, a policy learning problem is characterized by an arbitrary output space $\mathcal{O}$ together with a \emph{validator} $\mathcal{V}:\mathcal{O}\to\{0,1\}$ indicating whether an algorithm's output $\mathcal{A}(Z)\in\mathcal{O}$ is correct for the underlying instance $P$.

However, defining correctness in a uniform way across different problems with different output spaces is tricky. Na\"{i}vely, we might want to ask $\mathcal{A}$ to be accurate---i.e., 
\begin{align*}
\hat{\mathbb{P}}_{P,n}[\mathcal{V}(\mathcal{A}(Z))=1]\ge1-\delta(n),
\end{align*}
where $\hat{\mathbb{P}}_{P,n}$ is the probability measure of drawing $n$ samples according to the observational data distribution represented by $P$, and the error bound $\delta(n)$ is a function of the number of samples $n$. Then, we could invert $n$ to obtain a sample complexity for $\delta$. In this approach, an $\mathcal{A}$ might output the correct answer by chance. For instance, an algorithm $\mathcal{A}_{\text{imp}}$ for solving the improving policy problem for an instance $P$ might obtain the right answer from very few samples by chance, without realizing that its answer is correct; in contrast, correctly answering the existence problem for $P$ may require a large number of samples. Thus, for instance $P$, the improving policy problem might is easier than the existence problem according to this notion of correctness.

To avoid this issue, we want a notion of algorithms that only provide an output when they are confident. To this end, we modify algorithms to take as input a confidence level, and extend their output space to include a special symbol $??$ meaning that they abstain from answering. Then, we borrow the idea from statistics that \emph{valid} algorithms should only provide a positive answer when they are confident. Unlike hypothesis tests, where the output is always binary, we need to handle a variety of outputs ranging from binary to entire policies. To this end, we consider algorithms $\mathcal{A}(Z,\delta)\in\mathcal{O}\cup\{??\}$; i.e., they accept as input an observational dataset $Z$ and a confidence level $\delta\in\mathbb{R}_{>0}$ and either produce a desired output $o\in\mathcal{O}$ or abstain. Now, we define an algorithm $\mathcal{A}$ to be \emph{valid} if for all instances $P\in\mathcal{P}$ and for all $\delta\in\mathbb{R}_{>0}$, we have
\begin{align*}
\hat{\mathbb{P}}_{P,n}[\mathcal{A}(Z,\delta)=\;??\vee\mathcal{V}(\mathcal{A}(Z,\delta))=1]\ge1-\delta.
\end{align*}
In addition, we define the false negative rate (FNR) of $\mathcal{A}$ to be the probability such that the algorithm does not abstain---i.e., 
\begin{align*}
\hat{\mathbb{P}}_{P,n}[\mathcal{A}(Z,\delta)\neq\;??]\ge1-\delta.
\end{align*}
For simplicity, we use a single $\delta$ for both validity and FNR. With this general notion of correctness in place, we can define a reduction from one problem $(\mathcal{O}',\mathcal{V}')$ to another problem $(\mathcal{O},\mathcal{V})$ to be a mapping $\mathcal{R}:\mathcal{O}\to\mathcal{O}'$ such that $\mathcal{V}(o)=1\Rightarrow\mathcal{V}'(\mathcal{R}(o))=1$
for all $o\in\mathcal{O}$. In other words, if $o$ is correct for the second problem, then $\mathcal{R}(o)$ is correct for the first problem. It is easy to check that given a valid algorithm $\mathcal{A}$ for the second problem, then the algorithm $\mathcal{A}'(Z,\delta)=\mathcal{R}(\mathcal{A}(Z,\delta))$ is a valid algorithm for the first problem. This notion of reduction straightforwardly translates upper and lower bounds on sample complexities between the two problems.

Next, we establish straightforward reductions from the policy existence problem to the improving policy problem, and from the improving policy problem to the optimal policy problem. The more technically interesting question is regarding the opposite relationship---can we show that these gaps are strict? The optimal policy problem is easily shown to have strictly higher sample complexity than the improving policy problem. Our main technical result establishes a partial gap between the improving policy problem and the policy existence problem. This result demonstrates the value in thinking about a broader range of policy learning problem beyond just the optimal policy and improving policy problems---in certain scenarios, we can provide useful answers even when producing an optimal or improving policy is impossible.

\section{Problem Formulation}

In this section, we formalize (1) \emph{policy learning instances}, each of which specifies a distribution over units within the potential outcomes framework (Section~\ref{sec:policylearninginstance}), and (2) \emph{policy learning problems} (Section~\ref{sec:policylearningproblem}), each of which asks to compute some kind of output for a given policy learning instance (e.g., compute the best policy). Our formalism essentially specializes the standard definition of sample complexity from learning theory~\citep{valiant1984theory} to our policy learning setting. While it is abstract, it enables us to define the notion of reductions between policy learning problems (Section~\ref{sec:policylearningreduction}), which is necessary for us to formalize the notion that one policy learning problem is ``harder'' than another. Finally, we formalize the problems described in Section~\ref{sec:intro} within our framework, and prove hierarchical relations demonstrating that some problems are at least as hard as others. We note that these reductions are straightforward; the key question is whether these relationships are \emph{strict}---i.e., there do not exist any reductions in the opposite directions. The remainder of our paper is dedicated to proving various strictness results about our hierarchical relationships.

\subsection{Policy Learning Instances}
\label{sec:policylearninginstance}

A policy learning instance captures the information necessary to define a distribution in the potential outcomes framework. We consider a finite set of unit types $x\in\mathcal{X}=[k]$, a binary treatment $t\in\mathcal{T}=\{0,1\}$, and an outcome $y\in\mathcal{Y}=\mathbb{R}$. A policy learning instance $P=(\vec{\mu},\vec{\sigma})$ (where $\vec{\mu},\vec{\sigma}\in\mathbb{R}^{2\times k}$) defines a random variable $(X,Y_0,Y_1)\sim\mathbb{P}_P$ over $\mathcal{X}\times\mathcal{Y}\times\mathcal{Y}$, where the distribution over unit types $\mathbb{P}_P[X=x]=k^{-1}$ is uniform, and the distribution of the potential outcome under treatment $t\in\mathcal{T}$ conditioned on the unit having type $x\in\mathcal{X}$ is $\mathbb{P}_P[Y_t=y\mid X=x]=\mathcal{N}(y;\mu_{x,t},\sigma_{x,t}^2)$. We assume the uniform distribution over types to simplify our analysis, and it is not essential. We denote the space of all policy learning instances for a given number of types $k$ by $\mathcal{P}_k$, and denote the space of all instances by $\mathcal{P}=\bigcup_{k=1}^\infty\mathcal{P}_k$. Formalizing the space of all instances in this way enables us provide a general definition for an algorithm accepting policy learning instances as input.

\subsection{Treatment Assignment Policies}

The goal of policy learning is to learn a policy for assigning treatments to units depending on their type. Given an instance $P$, a \emph{policy} is a mapping $\pi:\mathcal{X}\to[0,1]$. We denote the space of policies over $k$ types by $\Pi_k$; note that a policy can be encoded as a $k$-dimensional vector, so $\Pi_k=[0,1]^k$. We let $\Pi=\bigcup_{k=1}^\infty\Pi_k$ denote the space of all policies. The \emph{value} of a policy $\pi$ for instance $P$ is $J(\pi;P)=k^{-1}\sum_{x\in\mathcal{X}}\mu_{x,\pi(x)}$. Now, a policy $\pi$ is \emph{optimal} for instance $P$ if $J(\pi)=\sup_{\pi'\in\Pi_k}J(\pi')$; we denote the subspace of optimal policies by $\Pi_P^{\text{opt}}\subseteq\Pi$. For instance, the policy $\pi^*_P(x)=\mathbf{1}(\tau_x\ge0)$ is optimal, where $\tau_x=\mu_{1,x}-\mu_{0,x}$ is the \emph{treatment effect}.

Next, we are interested in whether we can obtain policy improvements compared to constant treatment assignments rather than trying to obtain the optimal policy. Critically, we allow the algorithm to focus on a subset of unit types rather than needing to assign treatments for every type. Define $\pi_t(x)=t$ (for $t\in\mathcal{T}$) to be the constant policy that always outputs $t$; we denote the set of constant policies by $\Pi^{\text{const}}=\{\pi_t\mid t\in\mathcal{T}\}$. We consider \emph{partial policies} $\pi:\mathcal{X}\to\mathcal{T}\cup\{-1\}$, where $\pi(x)=-1$ means that $\pi$ is abstaining from assigning a treatment for type $x$; we let $\Pi^0$ denote the set of partial policies. Note that we can straightforwardly convert any policy $\pi\in\Pi$ to a partial policy $\pi\in\Pi^0$; we implicitly make this conversion and write $\Pi\subseteq\Pi^0$. Given a \emph{baseline treatment} $t\in\mathcal{T}$, we define the \emph{completion} $C_{\pi,t}\in\Pi$ of $\pi$ for $t$ by
\begin{align*}
C_{\pi,t}(x)=\begin{cases}
\pi(x)&\text{if }\pi(x)\neq-1 \\
t&\text{otherwise}.
\end{cases}
\end{align*}
In other words, $\bar{\pi}$ assigns a baseline treatment $t$ when $\pi$ abstains.
\begin{definition}
\rm
$\pi\in\Pi^0$ is \emph{improving} (denoted $\Pi_P^{\text{imp}}\subseteq\Pi^0$) if $J(C_{\pi,t})>J(\pi_t)$ for all $t\in\mathcal{T}$.
\end{definition}
Intuitively, the idea is that we want $\pi$ to outperform the constant policy $\pi_t$, but where we use $\pi_t$ instead of $\pi$ if $\pi$ abstains. Thus, we are always better off using $\pi$ than using a constant policy---if the baseline is to use constant policy $\pi_t$, then we are better off using $C_{\pi,t}$.


\subsection{Policy Learning Problems}
\label{sec:policylearningproblem}

Next, we turn to defining a general notion of a policy learning problem. The most common problem studied in the literature is the problem of computing an optimal policy, but recent work has also studied the problem of computing a policy that exhibits confident improvement~\citep{chernozhukov2025policy,bastani2025beatingwinnerscurseinferenceaware}; in addition, we are also interested in the problem of whether there exists a policy that is better than the baseline of assigning the same treatment to all units, and whether there exists heterogeneity in potential outcomes. Formalizing a general notion of policy learning problems enables us to formalize the notion that one problem is ``harder'' than another.

We assume that all policy learning problems have access to the same information as inputs. First, we assume they are given the standard deviations $\vec{\sigma}$, so only the means $\vec{\mu}$ need to be estimated. Second, we assume they are given a dataset of observations that can be used to estimate $\vec{\mu}$. Specifically, given an instance $P\in\mathcal{P}_k$ and a hyperparameter $n\in\mathbb{N}$, we consider a balanced sample of observations $Z=\{(x_i,t_i,Y_i)\}_{i=1}^{2kn}$, where each type-treatment pair $(x,t)\in\mathcal{X}\times\mathcal{T}$ appears exactly $n$ times, and $Y_i\sim\mathcal{N}(\mu_{x_i,t_i},\sigma_{x_i,t_i}^2)$ are independent samples. We consider a balanced sample to avoid the need to reason about random sample sizes for different type-treatment pairs. We denote the space of all possible samples $Z$ by $\hat{\mathcal{P}}=\bigcup_{k=1}^\infty\bigcup_{n=1}^\infty\hat{\mathcal{P}}_{k,n}$, where $\hat{\mathcal{P}}_{k,n}=(\mathcal{X}\times\mathcal{T}\times\mathcal{Y})^{2kn}$. Furthermore, we let $\hat{\mathbb{P}}_{P,n}$ denote the distribution of random samples $Z\in\hat{\mathcal{P}}_{k,n}$ for instance $P\in\mathcal{P}_k$. Note that by definition, only the outcomes $Y_i$ are random variables.

Now, a policy learning problem is defined by (1) a space $\mathcal{O}$ of desired outputs, and (2) a \emph{validator} $\mathcal{V}$ that checks whether an algorithm designed to solve that problem produces the correct output. For example, for the problem of learning the optimal policy, $\mathcal{O}$ would be the space of policies, and $\mathcal{V}$ would check whether the algorithm outputs the optimal policy for a given instance. The space $\mathcal{O}$ is specified by the policy learning problem. An important aspect of our formulation is that algorithms are allowed to \emph{abstain}---i.e., they should only produce an output if they are confident; otherwise, they return a special symbol $??$. In particular, we consider policy learning algorithms of the form $\mathcal{A}:\mathbb{R}^{k\times2}\times\hat{\mathcal{P}}\times\mathbb{R}_{>0}\to\mathcal{O}\cup\{??\}$; i.e., the algorithm takes as in put the standard deviations $\vec{\sigma}$, a sample $Z$, and an error bound $\delta\in\mathbb{R}_{>0}$, and produces a desired output $\mathcal{A}(\vec{\sigma},Z,\delta)\in\mathcal{O}\cup\{??\}$.

Then, a validator is a function   $\mathcal{V}:\mathcal{P}\times\mathcal{O}\to\mathbb{B}$ (with $\mathbb{B}=\{0,1\}$) such that $\mathcal{V}(P,o)$ indicates whether output $o$ is correct for instance $P$. Now, we formalize our policy learning problems:
\begin{itemize}
\item \textbf{Optimal policy problem:} The goal is to output an optimal policy; the output space is $\mathcal{O}_{\text{opt}}=\Pi$ and the validator is $\mathcal{V}_{\text{opt}}(\pi,P)=\mathbf{1}(\pi\in\Pi_P^{\text{opt}})$.
\item \textbf{Improving policy problem:} 
The goal is to output an improving (partial) policy, or a special symbol $0$ if none exists (i.e., $\Pi_P^{\text{imp}}=\varnothing$).\footnote{Unlike $??$, which indicates that the algorithm is not confident in its output, $0$ indicates that the algorithm is confident that no improving policy exists.}
Specifically, $\mathcal{O}_{\text{imp}}=\Pi^0\cup\{0\}$ and
\begin{align*}
\mathcal{V}_{\text{imp}}(\pi,P)=
\begin{cases}
\mathbf{1}(\pi=0)&\text{if }\Pi_P^{\text{imp}}=\varnothing \\
\mathbf{1}(\pi\in\Pi_P^{\text{imp}})&\text{otherwise}.
\end{cases}
\end{align*}
\item \textbf{Policy existence problem:} The goal is to determine if an improving policy exists; specifically, $\mathcal{O}_{\text{exist}}=\mathbb{B}$ and $\mathcal{V}_{\text{exist}}(o,P)=\mathbf{1}(o=\mathbf{1}(\Pi_P^{\text{imp}}\neq\varnothing))$.
\end{itemize}

\subsection{Sample Complexity}

Given a policy learning problem defined by a validator $\mathcal{V}$, our goal is to characterize how many samples $n$ are required for an algorithm $\mathcal{A}$ to achieve a desired false negative rate (FNR) at a given false positive rate (FPR) $\delta$. First, given an algorithm $\mathcal{A}$ and a validator $\mathcal{V}$, the FPR of $\mathcal{A}$ is
\begin{align*}
\text{FPR}_n(\mathcal{A};\mathcal{V},P,\delta)
&=\hat{\mathbb{P}}_{P,n}[\mathcal{A}(\vec{\sigma},Z,\delta)\neq\;??\wedge\mathcal{V}(\mathcal{A}(\vec{\sigma},Z,\delta),P)=0].
\end{align*}
That is, the FPR is the probability that algorithm $\mathcal{A}$ does not abstain and fails the validator.
\begin{definition}
\rm
$\mathcal{A}$ is \emph{valid} if $\text{FPR}_n(\mathcal{A};\mathcal{V},P,\delta)\le\delta$ for all $n\in\mathbb{N}$, $P\in\mathcal{P}$, and $\delta\in\mathbb{R}_{>0}$.
\end{definition}
We restrict to valid algorithms. Next, given $\delta\in\mathbb{R}_{>0}$ and problem $P$, the FNR of $\mathcal{A}$ for $P$ is
\begin{align*}
\text{FNR}_n(\mathcal{A};\mathcal{V},P,\delta)=\hat{\mathbb{P}}_{P,n}[\mathcal{A}(\vec{\sigma},Z,\delta)=\;??].
\end{align*}
In other words, the FNR is the probability that $\mathcal{A}$ abstains. Now, given $\delta\in\mathbb{R}_{>0}$, the \emph{sample complexity} of $\mathcal{A}$ for $P$ is
\begin{align*}
n(\mathcal{A};\mathcal{V},P,\delta)=\min\{n\in\mathbb{N}\mid\text{FNR}_n(\mathcal{A};\mathcal{V},P,\delta)\le\delta\}.
\end{align*}
In other words, at FPR $\delta$, $\mathcal{A}$ also achieves an FNR of $\delta$. In addition, given a subset of instances $\mathcal{Q}\subseteq\mathcal{P}$, the sample complexity of $\mathcal{A}$ across all instances $P\in\mathcal{Q}$ is
\begin{align*}
n(\mathcal{A};\mathcal{V},\mathcal{Q},\delta)=\max\{n(\mathcal{A};\mathcal{V},P,\delta)\mid P\in\mathcal{Q}\}.
\end{align*}
In both cases, the sample complexity may be $\infty$. Finally, we can define the sample complexity of a problem in terms of the sample complexity across all possible algorithms.
\begin{definition}
\rm
A problem $\mathcal{V}$ has \emph{sample complexity upper bound} $n(\mathcal{Q};\delta)$ on instances $\mathcal{Q}$ if there exists an algorithm $\mathcal{A}$ such that $n(\mathcal{A};\mathcal{V},\mathcal{Q},\delta)\le n(\mathcal{Q};\delta)$ for all $\delta\in\mathbb{R}_{>0}$, and it has \emph{sample complexity lower bound} $n(\mathcal{Q};\delta)$ if for any $\mathcal{A}$, $n(\mathcal{A};\mathcal{V},\mathcal{Q},\delta)\ge n(\mathcal{Q};\delta)$ for all $\delta\in\mathbb{R}_{>0}$.
\end{definition}

\subsection{Problem Hierarchy}
\label{sec:policylearningreduction}

Now that we have a notion of sample complexity, we define one problem $\mathcal{V}$ to be harder than another one $\mathcal{V}'$ on instances $\mathcal{Q}$ if there is exists $n(\mathcal{Q};\delta)$ that is a lower bound for $\mathcal{V}$ and an upper bound for $\mathcal{V}'$. We can establish that one problem is harder than another via reductions.
\begin{definition}
\rm
Given policy learning problems $\mathcal{V}:\mathcal{O}\times\mathcal{P}\to\mathbb{B}$ and $\mathcal{V}':\mathcal{O}'\times\mathcal{P}\to\mathbb{B}$,
a \emph{reduction} from $\mathcal{V}'$ to $\mathcal{V}$
is a function $\mathcal{R}:\mathcal{O}\to\mathcal{O}'$ such that
\begin{align}
\label{eqn:reduction}
\mathcal{V}(o,P)=1\Rightarrow\mathcal{V}'(\mathcal{R}(o),P)=1
\qquad(\forall o\in\mathcal{O},P\in\mathcal{P}).
\end{align}
\end{definition}
In particular, \eqref{eqn:reduction} says that if $o=\mathcal{A}(\vec{\sigma},Z,\delta)$ solves $\mathcal{V}$, then $\mathcal{R}(o)$ solves $\mathcal{V}'$. Thus, if we have an algorithm $\mathcal{A}$ for $\mathcal{V}$, then the algorithm $\mathcal{A}'(\vec{\sigma},Z,\delta)=\mathcal{R}(\mathcal{A}(\vec{\sigma},Z,\delta))$ (where we define $\mathcal{R}(??)=\;??$) solves $\mathcal{V}'$. Intuitively, $\mathcal{V}$ is at least as hard as $\mathcal{V}'$; formally, reductions translate sample complexity bounds between $\mathcal{V}$ and $\mathcal{V}'$ in the following ways.
\begin{proposition}
\label{prop:reductionbounds}
If there is a reduction from $\mathcal{V}'$ to $\mathcal{V}$, then any sample complexity upper bound for $\mathcal{V}$ is a sample complexity upper bound for $\mathcal{V}'$, and any sample complexity lower bound for $\mathcal{V}'$ is a sample complexity lower bound for $\mathcal{V}$.
\end{proposition}
Finally, we have the following reductions for our problems of interest.
\begin{proposition}
\label{prop:basicreductions}
The function
\begin{align*}
\mathcal{R}_{\text{imp}\to\text{opt}}(\pi)=\begin{cases}
\pi&\text{if }\pi\not\in\Pi^{\text{const}} \\
0&\text{otherwise}.
\end{cases}
\end{align*}
is a reduction from $\mathcal{V}_{\text{imp}}$ to $\mathcal{V}_{\text{opt}}$, and $\mathcal{R}_{\text{exist}\to\text{imp}}(\pi)=\mathbf{1}(\pi\neq0)$ is a reduction from $\mathcal{V}_{\text{exist}}$ to $\mathcal{V}_{\text{imp}}$.
\end{proposition}
These results are straightforward so we omit proofs. Propositions~\ref{prop:reductionbounds} \&~\ref{prop:basicreductions} say that the optimal policy problem $\mathcal{V}_{\text{opt}}$ is at least as hard as the improving policy problem $\mathcal{V}_{\text{imp}}$, which is in turn at least as hard as the existence problem $\mathcal{V}_{\text{exist}}$.



\section{Theoretical Analysis}

In this section, we provide a theoretical analysis of the sample complexity of our three policy learning problems. First, in Section~\ref{sec:basic}, we provide a number of results on sample complexity upper and lower bounds for these problems; these results are all based on standard arguments, and are unsurprising. Our main  contribution comes in Section~\ref{sec:gap}, where we both provide a discussion of the gap between the optimal policy and improving policy problems, and most notably, a partial gap between the improving policy and policy existence problems.

\subsection{Basic Results}
\label{sec:basic}

\begin{algorithm}[t]
\caption{Optimal Policy}
\label{alg:optimal}
\begin{algorithmic}
\Procedure{OptimalPolicy}{$\boldsymbol{\sigma}, Z, \delta$}
\For{each $x \in \mathcal{X}$}
\State $T_x^{\delta} \gets 
\sqrt{\frac{2\sigma_x^2 \log(2k/\delta)}{n}}$
\EndFor
\If{$\exists x \in \mathcal{X}$ such that $|\hat{\tau}_x| \le T_x^{\delta}$}
\State \textbf{return} $??$
\Else
\State \textbf{return} 
$\hat{\pi}(x) \gets \mathbf{1}\{\hat{\tau}_x \ge 0\}$
\EndIf
\EndProcedure
\end{algorithmic}
\end{algorithm}

We begin by establishing standard results on policy learning within our theoretical framework. Our first result provides a sample complexity upper bound for learning the optimal policy; this upper bound is based on the algorithm presented in Algorithm~\ref{alg:optimal}.
\begin{theorem}
\label{thm:optimalupper}
Given $\tau_{\text{min}},\sigma_{\text{max}}\in\mathbb{R}_{>0}$, consider the set of instances
\begin{align*}
\mathcal{Q}_k^{\text{opt}}(\tau_{\text{min}},\sigma_{\text{max}})=\{(\vec{\mu},\vec{\sigma})\in\mathcal{P}_k\mid\forall x\in\mathcal{X}\;.\;|\tau_x|\ge\tau_{\text{min}}\wedge\sigma_x\le\sigma_{\text{max}}\}.
\end{align*}
Then, the following is a sample complexity upper bound for the optimal policy problem:
\begin{align*}
n(\mathcal{Q}_k^{\text{opt}}(\tau_{\text{min}},\sigma_{\text{max}});\delta)
=\frac{8\sigma_{\text{max}}^2\log(2k/\delta)}{\tau_{\text{min}}^2}.
\end{align*}
\end{theorem}
We provide a proof in Appendix~\ref{sec:thm:optimalupper:proof}. In other words, we can learn the optimal policy when all the treatment effects are bounded away from zero and all their variances are bounded above. Intuitively, this means that for each type $x\in\mathcal{X}$, once we obtain sufficiently many samples, we will have high confidence as to whether $\tau_x>0$ or $\tau_x<0$. Indeed, this result follows straightforwardly from a concentration bound on $\tau_x$; we use Mill's inequality, but Hoeffding's inequality would also work. Then, Algorithm~\ref{alg:optimal} returns a policy $\hat\pi$ only if it is confident it is optimal based on our concentration bound from Mill's inequality; otherwise, it returns $??$.

Next, we prove a complementary result establishing a lower bound on the sample complexity for the optimal policy problem.
\begin{theorem}
\label{thm:optimallower}
Let $\mathcal{Q}_k^{\text{opt}}$ be as in Theorem \ref{thm:optimalupper}. Then, the following is a sample complexity lower bound for the optimal policy problem:
\begin{align*}
n(\mathcal{Q}^{\text{opt}}_k(\tau_{\text{min}},\sigma_{\text{max}});\delta)
=\frac{\sigma_{\max}^2\log(1/(6\delta))}{2\tau_{\min}^2}.
\end{align*}
\end{theorem}
We give a proof in Appendix~\ref{sec:thm:optimallower:proof}; it follows from a standard argument. This lower bound equals the upper bound in Theorem~\ref{thm:optimalupper} up to constants, so Algorithm~\ref{alg:optimal} is near-optimal for $\mathcal{Q}_k^{\text{opt}}(\tau_{\text{min}},\sigma_{\text{max}})$.

\begin{algorithm}[t]
\caption{Improving Policy}
\label{alg:improving}
\begin{algorithmic}
\Procedure{ImprovingPolicy}{$\boldsymbol{\sigma}, Z, \delta$}
\For{each $x \in \mathcal{X}$}
\State $T_x^{\delta} \gets 
\sqrt{\frac{2\sigma_x^2 \log(2k/\delta)}{n}}$
\State $\hat{\pi}(x) \gets
\begin{cases}
1, & \text{if } \hat{\tau}_x > T_x^{\delta}, \\
0, & \text{if } \hat{\tau}_x < -T_x^{\delta}, \\
-1, & \text{otherwise.}
\end{cases}$
\EndFor
\If{$\exists\, x_0, x_1 \in \mathcal{X}$ such that 
$\hat{\pi}(x_0)=0$ and $\hat{\pi}(x_1)=1$}
\State \textbf{return} $\hat{\pi}$
\ElsIf{$\hat{\pi} \in \Pi^{\mathrm{const}}$}
\State \textbf{return} $0$
\Else
\State \textbf{return} $??$
\EndIf
\EndProcedure
\end{algorithmic}
\end{algorithm}

Now, we turn our attention to the improving policy problem. We begin with a sample complexity upper bound for this problem, which is based on Algorithm~\ref{alg:improving}. This algorithm acts according to essentially the same principles as Algorithm~\ref{alg:optimal}, but it is relaxed to account for only needing to check existence. Just as Algorithm~\ref{alg:optimal}, it whether it is confident that $\tau_x>0$ or $\tau_x<0$ for each $x\in\mathcal{X}$ (using Mill's inequality); however, in this case, it returns as long as there exists \emph{some} $x_0,x_1\in\mathcal{X}$ such that it is confident that $\tau_{x_0}<0$ and that $\tau_{x_1}>0$.
\begin{theorem}
\label{thm:improvingupper}
Given $\tau_{\text{min}},\sigma_{\text{max}}\in\mathbb{R}_{>0}$, consider the set of instances $\mathcal{Q}_k^{\text{imp}}=\mathcal{Q}_{k}^+\cup\mathcal{Q}_k^-$, where
\begin{align*}
\mathcal{Q}_k^+(\tau_{\text{min}},\sigma_{\text{max}})
&=\{(\vec{\mu},\vec{\sigma})\in\mathcal{P}_k\mid(\exists x_0,x_1\in\mathcal{X}\;.\;\tau_{x_0}\le-\tau_{\text{min}}\wedge\tau_{x_1}\ge\tau_{\text{min}})
\wedge\forall x\in\mathcal{X}\;.\;\sigma_x^2\le\sigma_{\text{max}}^2\} \\
\mathcal{Q}_k^-(\tau_{\text{min}},\sigma_{\text{max}})
&=\{(\vec{\mu},\vec{\sigma})\in\mathcal{P}_k\mid(\forall x\in\mathcal{X}\;.\;\tau_x\ge\tau_{\text{min}}
\vee\forall x\in\mathcal{X}\;.\;\tau_x\le-\tau_{\text{min}})
\wedge\forall x\in\mathcal{X}\;.\;\sigma_x^2\le\sigma_{\text{max}}^2\}.
\end{align*}
Then, the following is a sample complexity upper bound for the improving policy problem:
\begin{align*}
n(\mathcal{Q}_k^{\text{imp}}(\tau_{\text{min}},\sigma_{\text{max}});\delta)
=\frac{8\sigma_{\text{max}}^2\log(2k/\delta)}{\tau_{\text{min}}^2}.
\end{align*}
\end{theorem}
We give a proof in Appendix~\ref{sec:thm:improvingupper:proof}; the proof is standard. Note that the sample complexity is very similar to the one in Theorem~\ref{thm:optimalupper}. Instead, the difference lies in the set of instances for which the sample complexity applies---in particular, we have $\mathcal{Q}_k^{\text{opt}}\subseteq\mathcal{Q}_k^{\text{imp}}$. The difference is not small---$\mathcal{Q}_k^{\text{imp}}$ only requires that there is \emph{some} type for which the algorithm is confident that the treatment is $t$ for each $t\in\mathcal{T}$, whereas $\mathcal{Q}_k^{\text{opt}}$ says that \emph{every} type must be confident.

Next, we provide a lower bound for the improving policy problem.
\begin{theorem}
\label{thm:improvinglower}
Let $\mathcal{Q}_k^{\text{imp}}$ be as in Theorem \ref{thm:improvingupper}.
Then, the following is a sample complexity lower bound for the improving policy problem:
\begin{align*}
n(\mathcal{Q}_k^{\text{imp}}(\tau_{\text{min}},\sigma_{\text{max}});\delta)
=\frac{\sigma_{\text{max}}^2\log(1/(6\delta))}{2\tau_{\text{min}}^2}.
\end{align*}
\end{theorem}
We give a proof in Appendix~\ref{sec:thm:improvinglower:proof}; again, the result is standard.

\begin{algorithm}[t]
\caption{Policy Existence}
\label{alg:existence}
\vskip6pt
\begin{algorithmic}
\Procedure{PolicyExistence}{$\boldsymbol{\sigma}, Z, \delta$}
\State $\hat{\pi} \gets$ \Call{ImprovingPolicy}{$\boldsymbol{\sigma}, Z, \delta/2$}
\If{$\hat{\pi} \neq\;??$}
\State \textbf{return} $\mathbf{1}\{\hat{\pi} \neq 0\}$
\EndIf
\State $\hat{S}_t \gets (nk)^{-1}\sum_{x\in\mathcal{X}}
\mathbf{1}\{\hat{\tau}_x>0\}\,\hat{s}_x^2$,
where $\hat{s}_x \gets \hat{\tau}_x/(\sigma_x/\sqrt{n})$
\State $S_t^{\delta} \gets n^{-1}\Bigl[\frac{1}{2}
+ 3\sqrt{k^{-1}\log(4/\delta)}
+ k^{-1}\log(4/\delta)\Bigr]$
\If{$\forall\, t \in \mathcal{T}$, $\hat{S}_t > S_t^{\delta}$}
\State \textbf{return} $1$
\Else
\State \textbf{return} $??$
\EndIf
\EndProcedure
\Statex
\end{algorithmic}
\end{algorithm}

Next, we turn our attention to the policy existence problem. Our algorithm for this problem is summarized in Algorithm~\ref{alg:existence}. To ensure its sample complexity is not worse than the sample complexity of Algorithm~\ref{alg:improving}, it simply calls Algorithm~\ref{alg:improving} and uses its result if it is confident. It then performs a hypothesis test based on the probability ratio test, which we discuss in Section~\ref{sec:gap}. For the purposes of this section, we obtain an upper bound on the sample complexity simply via the reduction to Algorithm~\ref{alg:improving}.
\begin{theorem}
\label{thm:existenceupper}
Let $\mathcal{Q}_k^{\text{exist}}$ be as in Theorem~\ref{thm:improvingupper}. Then, the following is a sample complexity upper bound for the policy existence problem:
\begin{align*}
n(\mathcal{Q}_k^{\text{exist}}(\tau_{\text{min}},\sigma_{\text{max}});\delta)=\frac{8\sigma_{\text{max}}^2\log(4k/\delta)}{\tau_{\text{min}}^2}.
\end{align*}
\end{theorem}
We provide a proof in Appendix~\ref{sec:thm:existenceupper:proof}. This sample complexity upper bound is essentially the same as the one in Theorem~\ref{thm:improvingupper} for the improving policy problem. Next, we provide a sample complexity lower bound.
\begin{theorem}
\label{thm:existencelower}
Let $\mathcal{Q}_k^{\text{exist}}$ be as in Theorem~\ref{thm:improvingupper}.
Then, the following is a sample complexity lower bound for the policy existence problem:
\begin{align*}
n(\mathcal{Q}_k^{\text{exist}}(\tau_{\text{min}},\sigma_{\text{max}});\delta)
=\frac{\sigma_{\max}^2\log(1/(6\delta))}{2\tau_{\min}^2}.
\end{align*}
\end{theorem}
We give a proof in Appendix~\ref{sec:thm:existencelower:proof}. This sample complexity lower bound is identical to the one in Theorem~\ref{sec:thm:improvinglower:proof} (because the reduction works the other way for lower bounds, Theorem~\ref{sec:thm:improvinglower:proof} is reduced to this one rather than vice versa).

\subsection{Partial Gap Between Improving Policy and Policy Existence Problems}
\label{sec:gap}

Thus far, most of the results exhibit similar sample complexities. The first gap between the optimal policy and improving policy problems is through the set of instances rather than the sample complexity itself. It is easy to see that the sample complexity of the optimal policy learning problem for $\mathcal{Q}_k^{\text{imp}}$ is infinite---i.e.,
\begin{align*}
n(\mathcal{Q}_k^{\text{imp}}(\tau_{\text{min}},\sigma_{\text{max}});\delta)=\infty,
\end{align*}
since $\mathcal{Q}_k^{\text{imp}}$ includes instances with types $x$ where $\tau_x=0$; we can never be confident about what treatment to assign for such types $x$. Thus, there is a clear gap between these two problems.

The gap between the improving policy and policy existence problems is more subtle. Our results in Section~\ref{sec:basic} essentially provide the same bound. However, we can in fact show that there are instances for which there is a substantial gap in sample complexity between Algorithm~\ref{alg:existence} and Algorithm~\ref{alg:improving}. To understand why, we return to the second part of Algorithm~\ref{alg:existence}, which uses a different strategy to check for policy existence. Intuitively, the test statistic $\hat{S}_t$ is the probability ratio test where the null hypothesis is that no improving policy exists. Intuitively, if all types have $\tau_x\le0$ (resp., $\tau_x\ge0$), then $\hat{S}_t$ becomes small as the number of samples becomes larger.

The critical difference is that this test aggregates samples across types. Thus, it can actually deduce existence even if the number of samples per type is very small (or even one), if there are a large number of types to compensate. As an intuitive example, consider an instance $P_k$ where $k$ is very large, and where most types $x\in\mathcal{X}_0$ have $\tau_x=-\tau$ (for some $\tau\in\mathbb{R}_{>0}$) and some small number $m=k^{2/3}$ of types $x\in\mathcal{X}_1\subseteq\mathcal{X}$ have $\tau_x=\tau$. Furthermore, suppose we have a single sample from each type (i.e., $n=1$). In this case, the best constant policy assigns $t=0$. Under the null hypothesis (i.e., all types have $\tau_x<0$), the expected number of samples where $\hat\tau_x\ge0$ is $k/2$, so by Hoeffding's inequality, the empirical number is $\le\text{const}\cdot(k+\sqrt{k})/2$ with high probability. For our instance, the expected number of samples where $\hat\tau_x\ge0$ is $(k/2)+m=(k/2)+k^{2/3}$, and the empirical number is $\ge\text{const}'\cdot((k/2)+k^{2/3}-\sqrt{k})$. For sufficiently large $k$, we can reject the null hypothesis.

However, we still need to show that the improving policy problem cannot be solved for $P_k$ with few samples. This is challenging in general, but we provide a partial answer---roughly speaking, we show that algorithms $\mathcal{A}$ satisfying a very reasonable condition require a sublinear polynomial number of samples to solve the improving policy problem for the instance $P_k$ with $k$ sufficiently large. The technical condition on $\mathcal{A}$ is the following:
\begin{definition}
\label{def:monotone}
\rm
An algorithm $\mathcal{A}:\mathbb{R}^{k\times2}\times\hat{\mathcal{P}}\times\mathbb{R}_{>0}\to\Pi^0\cup\{0\}\cup\{??\}$ for the improving policy problem is \emph{monotone} if for any input $(\vec{\sigma},Z,\delta)$ where $\mathcal{A}$ outputs $\hat\pi\in\Pi^0$, then $\hat\pi$ satisfies the following property: for all $x\in\mathcal{X}$ such that $\hat\pi(x)=1$, if $\hat\tau_{x'}/\sigma_{x'}\ge\hat\tau_x/\sigma_x$ for some $x'\in\mathcal{X}$, then $\hat\pi(x')=1$.
\end{definition}
This condition says that $\mathcal{A}$ assigns treatments to types based on the magnitude of the statistic $\hat\tau_x/\sigma_x$; this statistic is a normalized Gaussian, so if $\hat\tau_x/\sigma_x>\hat\tau_{x'}/\sigma_{x'}$, then the likelihood that that $\tau_x>0$ is strictly greater than the likelihood that $\tau_{x'}>0$. As a consequence, it is natural to expect $\mathcal{A}$ to assign treatment $t=1$ to $x$ if it assigns treatment $t=1$ to $x'$.

Returning to our instance $P_k$, intuitively, there are so many more types $x\in\mathcal{X}_0$ that they ``swamp out'' the positive signal from types $x\in\mathcal{X}_1$. Thus, with the monotonicity constraint from Definition~\ref{def:monotone}, if an algorithm $\mathcal{A}$ for the improving policy problem assigns treatment $t=1$ to a type $x_1\in\mathcal{X}_1$, then it must assign $t=1$ to at least some type $x_0\in\mathcal{X}_0$ as well. However, in this case, the resulting policy $\pi$ will no longer be improving, since the mistake from assigning $t=1$ to $x_0$ cancels the benefit from correctly assigning $t=1$ to $x_1$. Our next\ theorem formalizes this argument.
\begin{theorem}
\label{thm:gap}
Consider an instance $P_k$ with: (1) $m$ types $\mathcal{X}_1\subseteq\mathcal{X}$ such that $\tau_x=\tau$ and $\sigma_x=c_1\sigma$, (2) $k-m-L$ types $\mathcal{X}_0\subseteq\mathcal{X}$ where $\tau_x=-\tau$, $\sigma_x=c_0\sigma$, and $L>8\log(4/\delta)$, and (3) $L$ types $\mathcal{X}_{-1}\subseteq\mathcal{X}$ where $\tau_x=-\tau$ and $\sigma_x=c_{-1}\sigma$. Suppose $m=\lfloor k^{1-\epsilon}\rfloor$, $c_1=k^{\epsilon/4}$, $c_0=k^{2\epsilon}$, and $c_{-1}=k^{-(1/2)+(\epsilon/3)}$ for any $\epsilon<1/8$. Then, we have the following:
\begin{itemize}
\item \textbf{Upper bound for policy existence:} Let $\mathcal{A}$ denote Algorithm~\ref{alg:existence} for the policy existence problem; given any $\delta\in\mathbb{R}_{>0}$,
for sufficiently large $k$, we have sample complexity upper bound
\begin{align*}
n(\mathcal{A};\mathcal{V}_{\text{exist}},P_k,\delta)=1.
\end{align*}
\item \textbf{Lower bound for improving policy:} Given any (valid) monotone algorithm $\mathcal{A}$, for sufficiently large $k$, we have sample complexity lower bound

\graceedit{
\begin{align*}
n(\mathcal{A};\mathcal{V}_{\text{imp}},P_k,\delta)\ge \left(\frac{c_1\sigma}{\sqrt{\pi}\tau}\right)^2
\qquad(\forall\delta\in(0,1/4)).
\end{align*}

}
\end{itemize}
\end{theorem}
In other words, there exists a problem $P_k$ such that the sample complexity for the policy existence problem is $n=1$, whereas the sample complexity for any monotone algorithm for the improving policy problem is at least \graceedit{$O(k^{\epsilon/2})$} (where $\epsilon\in\mathbb{R}_{>0}$ is a constant). Thus, this result establishes a partial gap between the policy existence and improving policy problems.

\section{Proof of Theorem~\ref{thm:gap}}
\label{sec:gap:proofs}

We first prove a general setting under which a gap exists, but with a complicated condition on the instance $P$; Theorem~\ref{thm:gap} is a consequence of this general analysis. First, we prove a general concentration inequality that will help prove concentration of our test statistic.

\begin{lemma}
\label{lem:noncentralcensoredchisquared}
Let $Z=k^{-1}\sum_{i=1}^kI_iX_i^2$, where $X_i\sim\mathcal{N}(0,1)$ i.i.d. and $I_i=\mathbf{1}(X_i>0)$; denote the measure of $\{X_i\}_{i=1}^k$ by $\mathbb{P}=\mathcal{N}(0,1)^k$. Then, we have
\begin{align*}
\mathbb{P}\left[Z<\frac{1}{2}+3\sqrt{\frac{\log(2/\delta)}{k}}+\frac{\log(2/\delta)}{k}\right]
&\ge1-\delta \\
\mathbb{P}\left[Z>\frac{1}{2}-3\sqrt{\frac{\log(2/\delta)}{k}}\right]
&\ge1-\delta.
\end{align*}
\end{lemma}
\begin{proof}
Letting $Z_i=X_i^2$ note that $Z_i$ and $I_i$ are independent, since $X_i^2$ does not depend on the sign of $X_i$. Also, $Z_i\sim\chi_1^2$ follows a chi-squared distribution, and $I_i\sim\text{Bernoulli}(1/2)$ follows a Bernoulli distribution. Thus, letting $\mathbb{P}'=(\chi_1^2\times\text{Bernoulli}(1/2))^k$, then $\{(Z_i,I_i)\}_{i=1}^k$ has distribution $\mathbb{P}'$. Let $\vec{i}\in\mathbb{B}^k$ enumerate the possible values of $\vec{I}$; note that $\vec{I}\sim\text{Bernoulli}(1/2)^k$, so $\|\vec{I}\|_1\sim\text{Binomial}(k,1/2)$ follows a Binomial distribution. Then, given any $\epsilon\in\mathbb{R}$, we have
\begin{align*}
\mathbb{P}'[Z\ge\epsilon]
=\sum_{\ell=0}^k\text{Binomial}(\ell;k,1/2)\sum_{\substack{\vec{i}\in\mathbb{B}^k,\\\|\vec{i}\|_1=\;\ell}}\mathbb{P}'\left[Z\ge\epsilon\Bigm\vert\vec{I}=\vec{i}\right],
\end{align*}
where $\text{Binomial}(\ell;k,p)$ is the probability mass function of the Binomial distribution. By Lemma~\ref{lem:binom},
\begin{align*}
\sum_{\ell=h_{\delta}+1}^k\text{Binomial}(\ell;k,1/2)\le\frac{\delta}{2}
\qquad\text{where}\qquad
h_{\delta}=\left\lfloor\frac{k}{2}+\sqrt{\frac{k\log(2/\delta)}{2}}\right\rfloor,
\end{align*}
Thus, we have
\begin{align*}
\mathbb{P}'[Z\ge\epsilon]
\le\frac{\delta}{2}+\sum_{\ell=0}^{h_{\delta}}\text{Binomial}(\ell;k,1/2)\sum_{\substack{\vec{i}\in\mathbb{B}^k,\\\|\vec{i}\|_1=\;\ell}}\mathbb{P}'\left[Z\ge\epsilon\Bigm\vert\vec{I}=\vec{i}\right],
\end{align*}
Conditioned on $\vec{I}=\vec{i}$ with $\|\vec{i}\|_1=\ell$, $Z\sim\chi_{\ell}^2$ follows a chi-squared distribution, so by Lemma~\ref{lem:lm},
\begin{align*}
\mathbb{P}'\left[\frac{kZ}{\ell}\ge1+2\sqrt{\frac{\log(2/\delta)}{\ell}}+\frac{\log(2\delta)}{\ell}\Biggm\vert\vec{I}=\vec{i}\right]\le\frac{\delta}{2}.
\end{align*}
Thus, taking $\epsilon$ appropriately, by a union bound, we have
\begin{align*}
\mathbb{P}'\left[Z\ge\frac{h_{\delta}}{k}+\frac{2}{k}\sqrt{h_{\delta}\log(6/\delta)}+\frac{\log(2/\delta)}{k}\right]
\le\delta,
\end{align*}
and plugging in for $h_{\delta}$, we have
\begin{align*}
\mathbb{P}'\left[Z\ge\frac{1}{2}+3\sqrt{\frac{\log(2/\delta)}{k}}+\frac{\log(2/\delta)}{k}\right]
\le\delta.
\end{align*}
The second inequality follows by essentially the same argument.
\end{proof}
Next, we prove a helper lemma that bounds a Gaussian integral for the following result.
\begin{lemma}
\label{lem:stripexpectation}
Given $s\in\mathbb{R}$ and $r\in\{-1,1\}$, define
\begin{align*}
\Delta(s,r)
=r\int_0^sx^2e^{-(x-s)^2/2}dx.
\end{align*}
Then, we have
\begin{alignat*}{3}
&&\Delta(s)
&\ge\frac{|s|^3}{8}\cdot e^{-s^2/8}
&&\qquad\text{if }(r=1\wedge s>0)\vee(r=-1\wedge s\le0) \\
&&\Delta(\mu)
&\ge-|s|^3
&&\qquad\text{if }(r=1\wedge s\le0)\vee(r=-1\wedge s>0).
\end{alignat*}
\end{lemma}
\begin{proof}
If $r=1$ and $s\ge0$, then
\begin{align*}
\Delta(r,s)
\ge\int_{s/2}^sx^2e^{-(x-s)^2/2}dx
\ge\frac{s^3}{8}\cdot e^{-s^2/8}.
\end{align*}
If $r=1$ and $s<0$, then
\begin{align*}
\Delta(r,s)=-\int_s^0x^2e^{-(x-s)^2/2}dx
&\ge-s^3.
\end{align*}
The case $r=-1$ follows similarly.
\end{proof}
Next, we prove our key lemma, which says our test statistic concentrates to its mean.
\begin{lemma}
\label{lem:existencemain}
Define
\begin{align*}
\hat{S}_t&=\frac{1}{nk}\sum_{x\in\mathcal{X}}\hat\iota_{x,t}\hat{s}_x^2
\qquad\text{where}\qquad
\hat{s}_x=\frac{\hat\tau_x}{\sigma_x/\sqrt{n}},
\qquad
\hat\iota_{x,t}=\begin{cases}
\mathbf{1}(\hat\tau_x>0)&\text{if }t=1 \\
\mathbf{1}(\hat\tau_x<0)&\text{if }t=0,
\end{cases}
\end{align*}
as well as $s_x=\tau_x/(\sigma_x/\sqrt{n})$ and
\begin{align*}
\iota_{x,t}=\begin{cases}
1&\text{if }(t=1\wedge s_x>0)\vee(t=0\wedge s_x<0) \\
0&\text{if }(t=1\wedge s_x\le0)\vee(t=0\wedge s_x\ge0).
\end{cases}
\end{align*}
Then, we have the following:
\begin{itemize}
\item \textbf{Upper bound:} Suppose that $\iota_{x,t}=0$ for all $x\in\mathcal{X}$, and define
\begin{align*}
S_t^{\text{hi}}=\frac{1}{2n}+\frac{3}{n}\sqrt{\frac{\log(2/\delta)}{k}}+\frac{\log(2/\delta)}{nk}.
\end{align*}
Then, we have $\hat{\mathbb{P}}_{P,n}[\hat{S}_t<S_t^{\text{hi}}]\ge1-\delta$.
\item \textbf{Lower bound:} Define
\begin{align*}
S_{x,t}^{\text{lo}}&=\frac{1}{2}+r_ts_x\sqrt{\frac{2}{\pi}}
+\frac{s_x^2}{2}
+\begin{cases}
|s_x|^3e^{-s_x^2/8}/(8\sqrt{2\pi})&\text{if }\iota_{x,t}=1 \\
-|s_x|^3/\sqrt{2\pi}&\text{if }\iota_{x,t}=0
\end{cases},
\end{align*}
and $\tilde{\mathcal{X}}_t=\{x\in\mathcal{X}\mid S_{x,t}^{\text{lo}}>0\}$. Furthermore, define
$\nu_t=k^{-1}\sum_{x\in\tilde{\mathcal{X}}_t}(2|s_x|+s_x^2)^2$, and
\begin{align*}
S_t^{\text{lo}}&=\left(\frac{1}{nk}\sum_{x\in\tilde{\mathcal{X}}_t}S_{x,t}^{\text{lo}}\right)-\frac{3}{n}\sqrt{\frac{\log(4/\delta)}{k}}
-\frac{1}{n}\sqrt{\frac{2\nu_t\log(4/\delta)}{k}},
\end{align*}
Then, we have $\hat{\mathbb{P}}_{P,n}[\hat{S}_t>
S_t^{\text{lo}}]\ge1-\delta$.
\end{itemize}
\end{lemma}
\begin{proof} We prove each part separately.



\paragraph{Upper bound.}

Define $\tilde{s}_x=\hat{s}_x-s_x$ and
\begin{align*}
\tilde\iota_{x,t}=\begin{cases}
\mathbf{1}(\tilde{s}_x>0)&\text{if }t=1 \\
\mathbf{1}(\tilde{s}_x<0)&\text{if }t=0.
\end{cases}
\end{align*}
Since $\iota_{x,t}=0$ for all $x\in\mathcal{X}$, it is easy to check that $\hat\iota_{x,t}\hat{s}_x^2\le\tilde\iota_{x,t}\tilde{s}_x^2$ for all $x\in\mathcal{X}$. Thus, defining $\tilde{S}_t=k^{-1}\sum_{x\in\mathcal{X}}\tilde\iota_{x,t}\tilde{s}_x^2$, then by Lemma~\ref{lem:noncentralcensoredchisquared}, we have
\begin{align*}
\hat{\mathbb{P}}_{P,n}\left[\tilde{S}_t<\frac{1}{2n}+3\sqrt{\frac{\log(2/\delta)}{k}}+\frac{\log(2/\delta)}{k}\right]\ge1-\delta,
\end{align*}
as claimed.





\paragraph{Lower bound.}

First, note that $\hat{S}_t\ge\hat{S}_t'$, where $\hat{S}_t'=(nk)^{-1}\sum_{x\in\tilde{\mathcal{X}}_t}\hat{\iota}_{x,t}\hat{s}_x^2$. Next, note that
\begin{align*}
\hat\iota_{x,1}&=\mathbf{1}(\hat{s}_x>0)
=\mathbf{1}(\hat{s}_x>s_x)
+\begin{cases}
\mathbf{1}(0<\hat{s}_x\le s_x)&\text{if }s_x>0 \\
-\mathbf{1}(s_x<\hat{s}_x\le0)&\text{if }s_x\le0
\end{cases} \\
\hat\iota_{x,0}
&=\mathbf{1}(\hat{s}_x<0)
=\mathbf{1}(\hat{s}_x<s_x)
+\begin{cases}
-\mathbf{1}(0\le\hat{s}_x<s_x)&\text{if }s_x\ge0 \\
\mathbf{1}(s_x\le\hat{s}_x<0)&\text{if }s_x<0,
\end{cases}
\end{align*}
Thus, defining
\begin{align*}
\tilde\iota_{x,t}&=\begin{cases}
\mathbf{1}(\hat{s}_x>s_x)&\text{if }t=1 \\
\mathbf{1}(\hat{s}_x<s_x)&\text{if }t=0
\end{cases}
\qquad\text{and}\qquad
\tilde\iota_{x,t}'=\begin{cases}
\mathbf{1}(0<\hat{s}_x\le s_x)&\text{if }t=1\wedge s_x>0 \\
-\mathbf{1}(s_x<\hat{s}_x\le0)&\text{if }t=1\wedge s_x\le0 \\
-\mathbf{1}(0\le\hat{s}_x<s_x)&\text{if }t=0\wedge s_x\ge0 \\
\mathbf{1}(s_x\le\hat{s}_x<0)&\text{if }t=0\wedge s_x<0,
\end{cases}
\end{align*}
then we have $\hat\iota_{x,t}=\tilde{\iota}_{x,t}+\tilde\iota_{x,t}'$. Further defining $\tilde{s}_x=\hat{s}_x-s_x$, along with
\begin{align*}
\tilde{S}_t=\frac{1}{nk}\sum_{x\in\tilde{\mathcal{X}}_t}\tilde\iota_{x,t}\tilde{s}_x^2
\qquad\text{and}\qquad
\tilde{S}_t'=\frac{1}{nk}\sum_{x\in\tilde{\mathcal{X}}_t}(\tilde\iota_{x,t}s_x(2\tilde{s}_x+s_x)+\tilde\iota_{x,t}'(\tilde{s}_x+s_x)^2),
\end{align*}
then we have $\hat{S}_t'=\tilde{S}_t+\tilde{S}_t'$. Finally, we note that
\begin{align*}
\tilde\iota_{x,t}&=\begin{cases}
\mathbf{1}(\tilde{s}_x>0)&\text{if }t=1 \\
\mathbf{1}(\tilde{s}_x<0)&\text{if }t=0.
\end{cases}
\qquad\text{and}\qquad
\tilde\iota_{x,t}'=\begin{cases}
\mathbf{1}(-s_x<\tilde{s}_x\le0)&\text{if }t=1\wedge s_x>0 \\
-\mathbf{1}(0<\tilde{s}_x\le-s_x)&\text{if }t=1\wedge s_x\le0 \\
-\mathbf{1}(-s_x\le\tilde{s}_x<0)&\text{if }t=0\wedge s_x\ge0 \\
\mathbf{1}(0\le\tilde{s}_x<-s_x)&\text{if }t=0\wedge s_x<0.
\end{cases}
\end{align*}
This way, we have rewritten $\hat{S}_t'$ purely in terms of the random variables $\{\tilde{s}_x\}_{x\in\tilde{\mathcal{X}}_t}$ and the constants $n$, $\tilde{k}_t=|\tilde{\mathcal{X}}_t|$, and $\{s_x\}_{x\in\tilde{\mathcal{X}}_t}$. Note that $\{\tilde{s}_x\}_{x\in\tilde{\mathcal{X}}_t}$ are equal in distribution to i.i.d. standard Gaussian random variables; we denote their measure by $\tilde{\mathbb{P}}_{\tilde{k}_t}=\mathcal{N}(0,1)^{\tilde{k}_t}$. First, we bound $\tilde{S}_t$; by Lemma~\ref{lem:noncentralcensoredchisquared},
\begin{align*}
\tilde{\mathbb{P}}_k\left[\tilde{S}_t>\frac{1}{2n}-\frac{3}{n}\sqrt{\frac{\log(4/\delta)}{k}}\right]
\ge1-\frac{\delta}{2}.
\end{align*}
Next, we bound $\tilde{S}_t'$. Define $r_1=1$ and $r_0=-1$, and $\Delta(r,s)$ as in Lemma~\ref{lem:stripexpectation}. By Lemma~\ref{lem:stripexpectation},
\begin{align*}
\tilde{\mathbb{E}}_{\tilde{k}_t}[\tilde\iota_{x,t}'(\tilde{s}_x+s_x)^2]
=\frac{\Delta(r_t,s_x)}{\sqrt{2\pi}}
\ge\begin{cases}
|s_x|^3e^{-s_x^2/8}/(8\sqrt{2\pi})&\text{if }\iota_{x,t}=1 \\
-|s_x|^3/\sqrt{2\pi}&\text{if }\iota_{x,t}=0.
\end{cases}
\end{align*}
In addition, $\tilde{\mathbb{E}}_{\tilde{k}_t}[\tilde\iota_{x,t}\tilde{s}_x]=r_t/\sqrt{2\pi}$ and $\tilde{\mathbb{E}}_{\tilde{k}_t}[\tilde\iota_{x,t}]=1/2$. Thus, letting
\begin{align*}
\tilde{S}_{x,t}'&=\tilde\iota_{x,t}s_x(2\tilde{s}_x+s_x)+\tilde\iota_{x,t}'(\tilde{s}_x+s_x)^2 \\
S_{x,t}'&=
r_ts_x\sqrt{\frac{2}{\pi}}
+\frac{s_x^2}{2}
+\begin{cases}
|s_x|^3e^{-s_x^2/8}/(8\sqrt{2\pi})&\text{if }\iota_{x,t}=1 \\
-|s_x|^3/\sqrt{2\pi}&\text{if }\iota_{x,t}=0,
\end{cases}
\end{align*}
then $\tilde{\mathbb{E}}_{\tilde{k}_t}[\tilde{S}_{x,t}']
\ge S_{x,t}'$. Note that $\tilde\iota_{x,t}\tilde{s}_x$ is $1$-subgaussian, $\tilde\iota_{x,t}$ is $(1/2)$-subgaussian, and $\tilde\iota_{x,t}'(\tilde{s}_x+s_x)^2$ is $(s_x^2/2)$-subgaussian (since it is bounded in $[0,s_x^2]$); thus, $\tilde{S}_{x,t}'$ is $(2|s_x|+s_x^2)$-subgaussian. As a consequence, letting $S_t'=(nk)^{-1}\sum_{x\in\tilde{\mathcal{X}}_t}S_{x,t}'$, then by Lemma~\ref{lem:hoeffding}, we have
\begin{align*}
\tilde{\mathbb{P}}_{\tilde{k}_t}\left[\tilde{S}_t'>S_t'-\frac{1}{n}\sqrt{\frac{2\nu_t\log(2/\delta)}{k}}\right]\ge1-\frac{\delta}{2}.
\end{align*}
The claim follows.
%
%
%
\end{proof}
Our first corollary uses the previous result to establish validity of Algorithm~\ref{alg:existence}.
\begin{corollary}
\label{cor:existencevalid}
Algorithm~\ref{alg:existence} is valid.
\end{corollary}
\begin{proof}
Denote Algorithm~\ref{alg:existence} by $\mathcal{A}$. Consider any instance $P\in\mathcal{P}$. Suppose there exists $x_0,x_1\in\mathcal{X}$ such that $\tau_{x_0}<0$ and $\tau_{x_1}>0$, so an improving policy exists. Then, $\mathcal{A}$ only returns incorrectly if $\hat\pi=0$. By Theorem~\ref{thm:improvingupper}, this happens with probability at most $\delta/2$, so validity follows. Alternatively, suppose that $\tau_x<0$ (resp., $\tau_x>0$) for all $x\in\mathcal{X}$, so no improving policy exists. By Theorem~\ref{thm:improvingupper}, the probability that $\hat\pi\neq\;??\wedge\hat\pi\neq0$ is at most $\delta/2$. Furthermore, by the upper bound in Lemma~\ref{lem:existencemain}, $\hat{\mathbb{P}}_{P,n}[\hat{S}_0<S_0^{\delta}]\ge1-\delta/2$ (resp., $\hat{\mathbb{P}}_{P,n}[\hat{S}_1<S_1^{\delta}]\ge1-\delta/2$). By a union bound, $\mathcal{A}$ incorrectly returns $0$ with probability at most $\delta$.
\end{proof}
Our second corollary establishes a condition on instances $P\in\mathcal{P}$ under which the sample complexity of Algorithm~\ref{alg:existence} is $n=1$.
\begin{corollary}
\label{cor:existencealternate}
We use the notation in Lemma~\ref{lem:existencemain}. Letting $\tilde{\mathcal{Q}}_k^+=\bigcup_{t\in\mathcal{T}}\tilde{\mathcal{Q}}_{k,t}^+$, where
\begin{align*}
\tilde{\mathcal{Q}}_{k,t}^+ = \{P\in\mathcal{P}\mid S_t^{\text{lo}}\ge R\}
\qquad\text{where}\qquad
R=\frac{1}{2}+6\sqrt{\frac{\log(4/\delta)}{k}}+\sqrt{\frac{2\nu_t\log(4/\delta)}{k}}+\frac{\log(4/\delta)}{k},
\end{align*}
then $n(\tilde{\mathcal{Q}}_k^+;\delta)=1$ is a sample complexity upper bound for the policy existence problem.
\end{corollary}
\begin{proof}
We use the notation in Lemma~\ref{lem:existencemain} \& Corollary~\ref{cor:existencealternate}. Fix any $P\in\tilde{\mathcal{Q}}_{k,t}^+(\tau_{\text{min}},\sigma_{\text{max}})$, and let $n=1$. Letting $\tilde{k}_t=|\tilde{\mathcal{X}}_t|$ and
\begin{align*}
S_t^{\text{lo}}&=\frac{\tilde{k}_t}{2k}+S_t^{\text{lo}\prime}-3\sqrt{\frac{\log(4/\delta)}{k}}
-\sqrt{\frac{2\nu_t\log(4/\delta)}{k}},
\end{align*}
then by the lower bound in Lemma~\ref{lem:existencemain}, $\hat{\mathbb{P}}_{P,n}[\hat{S}_t>S_t^{\text{lo}}]\ge1-\delta$. Thus,
we have
\begin{align*}
S_t^{\text{lo}}-S_t^{\delta}
\ge 
S_t^{\text{lo}\prime}-\frac{1}{2}\left(1-\frac{\tilde{k}_t}{k}\right)-6\sqrt{\frac{\log(4/\delta)}{k}}-\sqrt{\frac{2\nu_t\log(4/\delta)}{k}}-\frac{\log(4/\delta)}{k}\ge0.
\end{align*}
Thus, $\hat{\mathbb{P}}_{P,n}[\hat{S}_t>S_t^{\delta}]\ge1-\delta$, so $\mathcal{A}$ correctly returns $1$ with probability at least $1-\delta$.
\end{proof}

Finally, with these results in hand, we prove each part of Theorem~\ref{thm:gap} separately.

\paragraph{Upper bound for policy existence.}

First, note that
\begin{align*}
S_{x,1}^{\text{lo}}=\frac{1}{2}+\begin{cases}
\dfrac{\tau^2}{2c_1^2\sigma^2}+\dfrac{\tau}{c_1\sigma}\sqrt{\dfrac{2}{\pi}}+\dfrac{\tau^3e^{-\tau^2/(8c_1^2\sigma^2)}}{8c_1^3\sigma^3\sqrt{2\pi}}&\text{if }x\in\mathcal{X}_1 \\
\dfrac{\tau^2}{2c_0^2\sigma^2}-\dfrac{\tau}{c_0\sigma}\sqrt{\dfrac{2}{\pi}}-\dfrac{\tau^3}{c_0^3\sigma^3\sqrt{2\pi}}&\text{if }x\in\mathcal{X}_0 \\
\dfrac{\tau^2}{2c_{-1}^2\sigma^2}-\dfrac{\tau}{c_{-1}\sigma}\sqrt{\dfrac{2}{\pi}}-\dfrac{\tau^3}{c_{-1}^3\sigma^3\sqrt{2\pi}}&\text{if }x\in\mathcal{X}_{-1}.
\end{cases}
\end{align*}
It is clear that for $k$ sufficiently large, we have $S_{x,-1}^{\text{lo}}\le0$ for $x\in \mathcal{X}_{-1}$, and $S_{x,1}^{\text{lo}},S_{x,0}^{\text{lo}}>0$ for $x\in \mathcal{X}_{1}\cup\mathcal{X}_{0}$, so
\begin{align*}
S_1^{\text{lo}}
&\ge\frac{k-L}{2k}+\frac{(k-L)\tau^2}{2kc_0^2\sigma^2}+\left(\frac{m}{kc_1}-\frac{1}{c_0}\right)\frac{\tau}{\sigma}\sqrt{\frac{2}{\pi}}+\left(\frac{me^{-\tau^2/(8c_1^2\sigma^2)}}{8kc_1^3}-\frac{1}{c_0^3}\right)\frac{\tau^3}{\sigma^3\sqrt{2\pi}}.
\end{align*}
Furthermore, since $c_1\le c_0$, we have
\begin{align*}
\nu_1
=\frac{m}{k}\left(\frac{2\tau}{c_1\sigma}+\frac{\tau^2}{c_1^2\sigma^2}\right)^2+\frac{k-m-L}{k}\cdot\left(\frac{2\tau}{c_0\sigma}+\frac{\tau^2}{c_0^2\sigma^2}\right)^2
&\le\left(\frac{2\tau}{c_1\sigma}+\frac{\tau^2}{c_1^2\sigma^2}\right)^2.
\end{align*}
Thus,
\begin{align*}
S_1^{\text{lo}}-R
&\ge
-\frac{L}{2k}+\frac{(k-L)\tau^2}{2kc_0^2\sigma^2}+\left(\frac{m}{kc_1}-\frac{1}{c_0}\right)\frac{\tau}{\sigma}\sqrt{\frac{2}{\pi}}+\left(\frac{me^{-\tau^2/(8c_1^2\sigma^2)}}{8kc_1^3}-\frac{1}{c_0^3}\right)\frac{\tau^3}{\sigma^3\sqrt{2\pi}} \\
&\qquad
-6\sqrt{\frac{\log(4/\delta)}{k}}-\left(\frac{2\tau}{c_1\sigma}+\frac{\tau^2}{c_1^2\sigma^2}\right)\sqrt{\frac{2\log(4/\delta)}{k}}-\frac{\log(4/\delta)}{k}.
\end{align*}
It is easy to check that for sufficiently large $k$, $S_1^{\text{lo}}\ge R$. Next, note that
\begin{align*}
S_{x,0}^{\text{lo}}=\frac{1}{2}+\begin{cases}
\dfrac{\tau^2}{2c_1^2\sigma^2}-\dfrac{\tau}{c_1\sigma}\sqrt{\dfrac{2}{\pi}}-\dfrac{\tau^3}{c_1^3\sigma^3\sqrt{2\pi}}&\text{if }x\in\mathcal{X}_1 \\
\dfrac{\tau^2}{2c_0^2\sigma^2}+\dfrac{\tau}{c_0\sigma}\sqrt{\dfrac{2}{\pi}}+\dfrac{\tau^3e^{-\tau^2/(8c_0^2\sigma^2)}}{8c_0^3\sigma^3\sqrt{2\pi}}&\text{if }x\in\mathcal{X}_0 \\
\dfrac{\tau^2}{2c_{-1}^2\sigma^2}+\dfrac{\tau}{c_{-1}\sigma}\sqrt{\dfrac{2}{\pi}}+\dfrac{\tau^3e^{-\tau^2/(8c_{-1}^2\sigma^2)}}{8c_{-1}^3\sigma^3\sqrt{2\pi}}&\text{if }x\in\mathcal{X}_{-1}.
\end{cases}
\end{align*}
Thus, we have
\begin{align*}
S_0^{\text{lo}}
\ge\left(\frac{k-m}{2k}\right)+\dfrac{L\tau^2}{2kc_{-1}^2\sigma^2}+\dfrac{L\tau}{kc_{-1}\sigma}\sqrt{\dfrac{2}{\pi}}+\dfrac{L\tau^3e^{-\tau^2/(8c_{-1}^2\sigma^2)}}{8kc_{-1}^3\sigma^3\sqrt{2\pi}},
\end{align*}
and
\begin{align*}
\nu_0
&=\frac{m}{k}\left(\frac{2\tau}{c_1\sigma}+\frac{\tau^2}{c_1^2\sigma^2}\right)^2+\frac{k-m-L}{k}\cdot\left(\frac{2\tau}{c_0\sigma}+\frac{\tau^2}{c_0^2\sigma^2}\right)^2+\frac{L}{k}\left(\frac{2\tau}{c_{-1}\sigma}+\frac{\tau^2}{c_{-1}^2\sigma^2}\right)^2 \\
&\le\left(\frac{2\tau}{c_1\sigma}+\frac{\tau^2}{c_1^2\sigma^2}\right)^2+\frac{L}{k}\left(\frac{2\tau}{c_{-1}\sigma}+\frac{\tau^2}{c_{-1}^2\sigma^2}\right)^2.
\end{align*}
Thus, we have
\begin{align*}
S_0^{\text{lo}}-R
&\ge-\frac{m}{2k}+\dfrac{L\tau^2}{2kc_{-1}^2\sigma^2}+\dfrac{L\tau}{kc_{-1}\sigma}\sqrt{\dfrac{2}{\pi}}+\dfrac{L\tau^3e^{-\tau^2/(8c_{-1}^2\sigma^2)}}{8kc_{-1}^3\sigma^3\sqrt{2\pi}} \\
&\qquad-6\sqrt{\frac{\log(4/\delta)}{k}}-\left(\frac{2\tau}{c_1\sigma}+\frac{\tau^2}{c_1^2\sigma^2}
+\frac{2\tau\sqrt{L}}{c_{-1}\sigma\sqrt{k}}+\frac{\tau^2\sqrt{L}}{c_{-1}^2\sigma^2\sqrt{k}}
\right)\sqrt{\frac{2\log(4/\delta)}{k}}-\frac{\log(4/\delta)}{k}
\end{align*}
It is easy to check that for sufficiently large $k$, $S_0^{\text{lo}}\ge R$. The claim follows by Corollary~\ref{cor:existencealternate}.

\paragraph{Lower bound for improving policy.}

Consider grouping the types (arbitrarily) into clusters $\mathcal{X}^1,...,\mathcal{X}^m$, where each cluster has exactly one type $x\in\mathcal{X}_1$ and $h=\lfloor(k-m-L)/m\rfloor$ types $x\in\mathcal{X}_0$ (the remaining types are ignored); in addition, let $x^j$ denote the unique type $x_1^j\in\mathcal{X}^j_1=\mathcal{X}^j\cap\mathcal{X}_1$, and let $\mathcal{X}_0^j=\mathcal{X}^j\cap\mathcal{X}_0$. Now, we prove that for sufficiently large $k$, if
\begin{align*}
n<\left(\frac{c_1\sigma}{\sqrt{\pi}\tau}\right)^2,
\end{align*}
then we have $\hat{\mathbb{P}}_{P_k,n}[E]\ge1/2$, where $E$ is the event that
\begin{align*}
\forall j\in[m]\;.\;\exists x_0^j\in\mathcal{X}_0^j\;.\;\frac{\hat\tau_{x_0^j}}{c_0\sigma}\ge\frac{\hat\tau_{x_1^j}}{c_1\sigma}.
\end{align*}
First, we show that the claim follows from this result. Specifically, since $\mathcal{A}$ is valid, we have $\hat{\mathbb{P}}_{P_k,n}[E']\ge1-\delta$, where $E'$ is the event that
\begin{align*}
\hat\pi=\;??\vee\hat\pi\in\Pi_{P_k}^{\text{imp}}.
\end{align*}
By a union bound, $\hat{\mathbb{P}}_{P_k,n}[E\wedge E']\ge1/2-\delta$. On event $E\wedge E'$, assuming $\hat\pi\in\Pi_{P_k}^{\text{imp}}$, we must have
\begin{align*}
\exists j\in[m]\;.\;\hat\pi(x_1^j)=1,
\end{align*}
since otherwise, we would have $J(C_{\hat\pi,0})\le J(\pi_0)$. Since $E'$ holds, by monotonicity of $\mathcal{A}$ we must have $\hat\pi(x_0^j)=1$; however, this implies that $\hat\pi\not\in\Pi_{P_k}^{\text{imp}}$, a contradiction. Thus, on event $E\wedge E'$, we must have $\hat\pi=\;??$. As a consequence, we have $\hat{\mathbb{P}}_{P_k,n}[\hat\pi=\varnothing]\ge(1/2)-\delta$, so we must have either $\delta\ge1/4$ or else the sample complexity of $\mathcal{A}$ is greater than $n$.

Now, we prove that $\hat{\mathbb{P}}_{P_k,n}[E]\ge1/2$. To this end, note that
\begin{align*}
\frac{\hat\tau_{x_0^j}}{c_0\sigma}-\frac{\hat\tau_{x_1^j}}{c_1\sigma}
\sim\mathcal{N}\left(-\frac{\tau}{c_0\sigma}-\frac{\tau}{c_1\sigma},\frac{2}{n}\right)
=\mathcal{N}\left(-\left(\frac{1}{c_1}+\frac{1}{c_0}\right)\frac{\tau}{\sigma},\frac{2}{n}\right).
\end{align*}
Thus, we have
\begin{align*}
\hat{\mathbb{P}}_{P_k,n}\left[\frac{\hat\tau_{x_0^j}}{c_0\sigma}\ge\frac{\hat\tau_{x_1^j}}{c_1\sigma}\right]
&=\frac{1}{2}-\int_0^{(c_1^{-1}+c_0^{-1})\tau/\sigma}\sqrt{\frac{n}{4\pi}}\exp\left(-\frac{nx^2}{4}\right)dx \\
&\ge\frac{1}{2}-\left(\frac{1}{c_1}+\frac{1}{c_0}\right)\frac{\tau}{\sigma}\sqrt{\frac{n}{4\pi}}\\
& \geq \frac{1}{2}-\frac{2}{c_1}\frac{\tau}{\sigma}\frac{c_1\sigma}{\sqrt{\pi}\tau}\sqrt{\frac{1}{4\pi}}\\
&\ge\frac{1}{2} - \frac{1}{\pi}
\end{align*}
where the last step follows from our assumption on $n$. Now, because our samples are independent, the probability that this relationship holds for some $x_0^j\in\mathcal{X}_0^j$ is
\begin{align*}
\hat{\mathbb{P}}_{P_k,n}\left[\exists x_0^j\in\mathcal{X}_0^j\;.\;\frac{\hat\tau_{x_0^j}}{c_0\sigma}\ge\frac{\hat\tau_{x_1^j}}{c_1\sigma}\right]
&\ge1-\left(\frac{1}{2} + \frac{1}{\pi}\right)^h,
\end{align*}
so by a union bound over $m$, we have
\begin{align*}
\hat{\mathbb{P}}_{P_k,n}\left[\forall j\in[m]\;.\;\exists x_0^j\in\mathcal{X}_0^j\;.\;\frac{\hat\tau_{x_0^j}}{c_0\sigma}\ge\frac{\hat\tau_{x_1^j}}{c_1\sigma}\right]
&\ge1-m\left(\frac{1}{2} + \frac{1}{\pi}\right)^h.
\end{align*}
This quantity is at least $1/2$ for $k$ sufficiently large, as claimed. \hfill $\square$

\section{Conclusion}

We have provided a novel mathematical framework for reasoning about the sample complexity of algorithms targeting different policy learning problems, as well as reasoning about the relationships between these problems. In addition, we have established sample complexity results for the optimal policy problem (studied by most of the existing literature), the improving policy problem (which has received recent attention), and the policy existence problem (which has received almost no attention). Our theoretical analysis suggests that in fact, the policy existence problem can provide valuable insights when the optimal policy and improving policy problems are infeasible to solve.

We leave a number of directions for future work; we highlight two. First, there are many sample complexity questions that we leave to future work to resolve. Second, there is wide scope for studying problems beyond ours, filling gaps that can enrich our ability to provide decision-makers more granular insights about the usefulness of policy learning in their application.

\bibliographystyle{abbrvnat}
\bibliography{references} 

\newpage
\appendix

\clearpage

\section{Proofs for Section~\ref{sec:basic}}

\subsection{Preliminaries}

\begin{lemma}
\label{lem:main}
Given $\delta\in\mathbb{R}_{>0}$, define $\hat\tau_x=\hat\mu_{x,1}-\hat\mu_{x,0}$, where
\begin{align*}
\hat\mu_{x,t}=\frac{1}{n}\sum_{i=1}^{2kn}Y_i\cdot\mathbf{1}(x_i=x\wedge t_i=t)
\qquad\qquad(\forall x\in\mathcal{X},t\in\mathcal{T}).
\end{align*}
First, let $E_{\delta}$ be the event $|\hat\tau_x-\tau_x|<T_x^{\delta}$ for all $x\in\mathcal{X}$, where
\begin{align*}
T_x^{\delta}=\sqrt{\frac{2\sigma_x^2\log(2k/\delta)}{n}}
\qquad\text{with}\qquad
\sigma_x=\sqrt{\sigma_{x,0}^2+\sigma_{x,1}^2}.
\end{align*}
Then, $\hat{\mathbb{P}}_{P,n}[E_{\delta}]\ge1-\delta$.
\end{lemma}
\begin{proof}
Note that
\begin{align*}
\hat\mu_{x,t}\sim\mathcal{N}\left(\mu_{x,t},\frac{\sigma_{x,t}^2}{n}\right)
\qquad\text{and}\qquad
\hat\tau_x\sim\mathcal{N}\left(\tau_x,\frac{\sigma_x^2}{n}\right).
\end{align*}
Thus, the result follows from Lemma~\ref{lem:mill} and a union bound.
\end{proof}

\subsection{Proof of Theorem~\ref{thm:optimalupper}}
\label{sec:thm:optimalupper:proof}

Denote Algorithm~\ref{alg:optimal} by $\mathcal{A}$. First, we prove that $\mathcal{A}$ is valid. Consider any instance $P\in\mathcal{P}$. On event $E_{\delta}$, it is easy to check that if $\mathcal{A}$ returns $\hat\pi$,
then $\hat\pi(x)=\mathbf{1}(\tau_x\ge0)$,
so $\hat\pi\in\Pi_P^{\text{opt}}$. By Lemma~\ref{lem:main}, we have $\text{FPR}_n(\mathcal{A};\mathcal{V}_{\text{opt}},P,\delta)\le\hat{\mathbb{P}}_{P,n}[\neg E_{\delta}]\le\delta$ so $\mathcal{A}$ is valid.

Now, we bound the sample complexity of $\mathcal{A}$ on $\mathcal{Q}_k^{\text{opt}}(\tau_{\text{min}},\sigma_{\text{max}})$. Fix any $P\in\mathcal{Q}_k^{\text{opt}}(\tau_{\text{min}},\sigma_{\text{max}})$ and
\begin{align*}
n\ge\frac{8\sigma_{\text{max}}^2\log(2k/\delta)}{\tau_{\text{min}}^2}.
\end{align*}
Then $T_x^{\delta}\le|\tau_{\text{min}}|/2$, so on event $E_{\delta}$, $|\hat\tau_x|>T_x^{\delta}$ for all $x\in\mathcal{X}$, so $\mathcal{A}$ does not return $??$. Thus, by Lemma~\ref{lem:main}, we have $\text{FNR}_n(\mathcal{A};\mathcal{V}_{\text{opt}},P,\delta)\le\hat{\mathbb{P}}_{Z,n}[\neg E_{\delta}]\le\delta$. The claim follows. \hfill $\square$

\subsection{Proof of Theorem~\ref{thm:optimallower}}
\label{sec:thm:optimallower:proof}

Consider $P\in\mathcal{Q}_k^{\text{opt}}(\tau_{\text{min}},\sigma_{\text{max}})$ such that $P = (\vec{\mu},\vec{\sigma}_{\max})$ and $\vec{\tau}=(\tau_{\min},-\tau_{\min}, -\tau_{\min}, \cdots)$. Assume that 
\begin{align*}
n < \frac{\sigma_{\max}^2\log(1/6\delta)}{2\tau_{\min}^2}.
\end{align*}
Consider an alternative instance $P' \in \mathcal{Q}_k^{\text{opt}}(\tau_{\text{min}},\sigma_{\text{max}})$ with $P' = (\vec{\mu'},\vec{\sigma}_{\max})$ such that for each $\tau'_x = -\tau_{\min}$ for all $x\in\mathcal{X}$.
The relative entropy between $\hat{\mathbb{P}}_{P, n}$ and $\hat{\mathbb{P}}_{P', n}$ is 
\begin{align*}
D(\hat{\mathbb{P}}_{P, n}, \hat{\mathbb{P}}_{P', n}) & = \sum_{i = 1}^k n D(\mathcal{N}(\tau_i,\sigma_i^2), \mathcal{N}(\tau'_i,{\sigma_i}^2)) \\
& = n \sum_{i = 1}^k \frac{(\tau_i - \tau'_i)^2}{2\sigma_i^2} \\
& \leq(\frac{\sigma_{\max}^2}{2\tau_{\min}^2}) (\log \frac{1}{6\delta})\sum_{i = 1}^k  \frac{(\tau_i - \tau'_i)^2}{2\sigma_i^2}\\
& = (\frac{\sigma_{\max}^2}{2\tau_{\min}^2}) (\log \frac{1}{6\delta})\frac{2\tau_{\min}^2}{\sigma_{\max}^2}\\
& = \log \frac{1}{6\delta}.
\end{align*}

Now suppose $\mathcal{A}$ is any valid policy that aims to output an optimal policy such that $\text{FNR}_n(\mathcal{A}; \mathcal{V}_{\text{opt}}, P', \delta)<\delta$.  Let $E$ be the event that 
\begin{align*}
\mathcal{A}(\vec{\sigma},Z,\delta)\neq\;??\wedge  \mathcal{A}(\vec{\sigma},Z,\delta)\in \Pi_{P'}^{\text{opt}}.
\end{align*}
Then $E^c$ is 
\begin{align*}
\mathcal{A}(\vec{\sigma},Z,\delta)=\;??\vee  \mathcal{A}(\vec{\sigma},Z,\delta)\notin \Pi_{P'}^{\text{opt}}.
\end{align*}
Next, by the Bretagnolle-Huber inequality, 
\begin{align*}
\hat{\mathbb{P}}_{P, n}[E] + \hat{\mathbb{P}}_{P', n}[E^c]& \geq \frac{1}{2} \exp{(-D(\hat{\mathbb{P}}_{P, n}, \hat{\mathbb{P}}_{P', n}))}\\
& \geq  \frac{1}{2} \exp{(\log 6\delta)}\\
& = 3\delta.
\end{align*}
Since $ \text{FPR}_n(\mathcal{A}; \mathcal{V}_{\text{opt}}, P, \delta)\geq \hat{\mathbb{P}}_{P, n}[E]$ and $\text{FNR}_n(\mathcal{A}; \mathcal{V}_{\text{opt}}, P', \delta)\leq \delta$, we have 
\begin{align*}
\text{FPR}_n(\mathcal{A}; \mathcal{V}_{\text{opt}}, P, \delta) + \text{FNR}_n(\mathcal{A}; \mathcal{V}_{\text{opt}}, P', \delta) + \text{FPR}_n(\mathcal{A}; \mathcal{V}_{\text{opt}}, P', \delta)
& \geq  3\delta\\
\text{FPR}_n(\mathcal{A}; \mathcal{V}_{\text{opt}}, P, \delta) + \text{FPR}_n(\mathcal{A}; \mathcal{V}_{\text{opt}}, P', \delta)
& \geq  2\delta.
\end{align*}
This implies that 
\begin{align*}
\max \{ \text{FPR}_n(\mathcal{A}; \mathcal{V}_{\text{opt}}, P, \delta), \text{FPR}_n(\mathcal{A}; \mathcal{V}_{\text{opt}}, P', \delta)\} \geq \delta.
\end{align*}
This means that no matter how we choose the algorithm $\mathcal{A}$, $\mathcal{A}$ cannot be valid with FNR smaller than $\delta$ if  $n<\frac{\sigma_{\max}^2\log(1/6\delta)}{2\tau_{\min}^2}$. Therefore, we must have
\begin{align*}
n(\mathcal{A}; \mathcal{V}_{\text{opt}}, \mathcal{Q}^{\text{opt}}, \delta) \geq \frac{\sigma_{\max}^2 \log (1/6\delta)}{2\tau_{\min}^2}.
\end{align*}
The claim follows. \hfill $\square$

\subsection{Proof of Theorem~\ref{thm:improvingupper}}
\label{sec:thm:improvingupper:proof}

Denote Algorithm~\ref{alg:improving} by $\mathcal{A}$. First, we prove that $\mathcal{A}$ is valid. Consider any instance $P\in\mathcal{P}$. On event $E_{\delta}$, it is easy to check that if $\hat\pi(x)=1$, then $\tau_x>0$, and if $\hat\pi(x)=0$, then $\tau_x<0$ (otherwise, $\hat\pi(x)=-1$). Thus, if there exists $x_0,x_1\in\mathcal{X}$ such that $\hat\pi(x_0)=0$ and $\hat\pi(x_1)=1$, then it is clear that $J(C_{\hat\pi,t})-J(\pi_t)>0$ for all $t\in\mathcal{T}$; that is, $\hat\pi\in\Pi_P^{\text{imp}}$. Alternatively, suppose that $\hat\pi\in\Pi^{\text{const}}$, i.e., $\hat\pi=\pi_t$ for some $t\in\mathcal{T}$; then, $\pi_t$ must be optimal. Thus, on event $E_{\delta}$, $\mathcal{A}$ either returns $??$ or passes the validator, so by Lemma~\ref{lem:main}, $\text{FPR}_n(\mathcal{A};\mathcal{V}_{\text{imp}},P,\delta)\le\hat{\mathbb{P}}_{P,n}[\neg E_{\delta}]\le\delta$.

Now, we bound the sample complexity of $\mathcal{A}$ on $\mathcal{Q}_k^{\text{imp}}(\tau_{\text{min}},\sigma_{\text{max}})$. Fix any $P\in\mathcal{Q}_k^{\text{imp}}(\tau_{\text{min}},\sigma_{\text{max}})$ and
\begin{align*}
n\ge\frac{8\sigma_{\text{max}}^2\log(2k/\delta)}{\tau_{\text{min}}^2}.
\end{align*}
Then $T_x^{\delta}\le|\tau_{\text{min}}|/2$, so on event $E_{\delta}$, $|\hat\tau_x|>T_x^{\delta}$ for all $x\in\mathcal{X}$, so $\mathcal{A}$ does not return $??$. Thus, by Lemma~\ref{lem:main}, we have $\text{FNR}_n(\mathcal{A};\mathcal{V}_{\text{imp}},P,\delta)\le\hat{\mathbb{P}}_{Z,n}[\neg E_{\delta}]\le\delta$. The claim follows. \hfill $\square$

\subsection{Proof of Theorem~\ref{thm:improvinglower}}
\label{sec:thm:improvinglower:proof}

This result follows from Theorem~\ref{thm:existencelower} and Propositions~\ref{prop:reductionbounds} \&~\ref{prop:basicreductions}.

\subsection{Proof of Theorem~\ref{thm:existenceupper}}
\label{sec:thm:existenceupper:proof}

By Corollary~\ref{cor:existencevalid}, $\mathcal{A}$ is valid. Then, this result follows from Theorem~\ref{thm:improvingupper}, accounting for the fact that Algorithm~\ref{alg:existence} runs Algorithm~\ref{alg:improving} with $\delta/2$. \hfill $\square$

\subsection{Proof of Theorem~\ref{thm:existencelower}}
\label{sec:thm:existencelower:proof}

Consider $P\in\mathcal{Q}_k^{\text{imp}}(\tau_{\text{min}},\sigma_{\text{max}})$ with $P = (\vec{\mu},\vec{\sigma}_{\max})$ such that $\tau_1=\tau_{\text{min}}$ and $\tau_x=-\tau_{\min}$ for all $x\in\mathcal{X}\setminus\{1\}$. Assume that
\begin{align*}
n < \frac{\sigma_{\max}^2 \log (1/6\delta)}{2\tau_{\min}^2}.
\end{align*}
Consider an alternative instance $P' \in \mathcal{Q}_k^{\text{imp}}(\tau_{\text{min}},\sigma_{\text{max}})$ with $P' = (\vec{\mu}',\vec{\sigma}_{\max})$ 
such that $\tau_x'=-\tau_{\min}$ for all $x\in \mathcal{X}$.
The relative entropy between $\hat{\mathbb{P}}_{P, n}$ and $\hat{\mathbb{P}}_{P', n}$ is 
\begin{align*}
D(\hat{\mathbb{P}}_{P, n}, \hat{\mathbb{P}}_{P', n}) & = n \sum_{i = 1}^k  D(\mathcal{N}(\tau_i,\sigma_i^2), \mathcal{N}(\tau'_i,{\sigma_i}^2)) \\
& = n \sum_{i = 1}^k \frac{(\tau_i - \tau'_i)^2}{2\sigma_i^2} \\
& \leq (\frac{\sigma_{\max}^2}{2\tau_{\min}^2})(\log \frac{1}{6\delta} ) (\frac{2\tau_{\min}^2}{\sigma_{\max}^2})\\
& = \log \frac{1}{6\delta}.
\end{align*}
Now suppose $\mathcal{A}$ is any valid policy that aims to output an optimal policy such that $\text{FNR}_n(\mathcal{A}; \mathcal{V}_{\text{opt}}, P', \delta)<\delta$. Let $E$ be the event that $\mathcal{A}(\vec{\sigma},Z,\delta)=0$. Then $E^c$ is the event 
\begin{align*}
\mathcal{A}(\vec{\sigma},Z,\delta)=\;?? \vee \mathcal{A}(\vec{\sigma},Z,\delta)=\hat \pi.
\end{align*}
Hence, by the Bretagnolle-Huber inequality, 
\begin{align*}
\hat{\mathbb{P}}_{P, n}[E] + \hat{\mathbb{P}}_{P', n}[E^c]& \geq \frac{1}{2} \exp{(-D(\hat{\mathbb{P}}_{P, n}, \hat{\mathbb{P}}_{P', n}))}\\
& \geq  \frac{1}{2} \exp{(\log 6\delta)}\\
& = 3\delta.
\end{align*}
Note that $\text{FPR}_n(\mathcal{A}; \mathcal{V}_{\text{exist}}, P, \delta) \geq \hat{\mathbb{P}}_{P, n}[E]$ and $\text{FNR}_n(\mathcal{A}; \mathcal{V}_{\text{exist}}, P', \delta)<\delta$. We have 
\begin{align*}
\text{FPR}_n(\mathcal{A}; \mathcal{V}_{\text{exist}}, P, \delta) + \text{FNR}_n(\mathcal{A}; \mathcal{V}_{\text{exist}}, P', \delta) + \text{FPR}_n(\mathcal{A}; \mathcal{V}_{\text{exist}}, P', \delta) &\geq  3\delta\\
\text{FPR}_n(\mathcal{A}; \mathcal{V}_{\text{exist}}, P, \delta) +  \text{FPR}_n(\mathcal{A}; \mathcal{V}_{\text{exist}}, P', \delta) 
& \geq 2\delta.
\end{align*}
This implies that
\begin{align*}
\max \{ \text{FPR}_n(\mathcal{A}; \mathcal{V}_{\text{exist}}, P, \delta), \text{FPR}_n(\mathcal{A}; \mathcal{V}_{\text{exist}}, P', \delta)\} \geq \delta.
\end{align*}
This means that no matter how we choose the algorithm $\mathcal{A}$, $\mathcal{A}$ cannot be valid with this many of samples. Therefore, we must have 
\begin{align*}
n(\mathcal{A}; \mathcal{V}_{\text{exist}}, \mathcal{Q}_k^{\text{imp}}, \delta) \geq \frac{\sigma_{\max}^2 \log (1/6\delta)}{2\tau_{\min}^2}.
\end{align*}
The claim follows. \hfill $\square$

\section{Helper Lemmas}

\begin{lemma}[Mill's Inequality]
\label{lem:mill}
Given $X\sim\mathcal{N}(\mu,\sigma^2)$, for any $\epsilon\in\mathbb{R}_{>0}$, we have
\begin{align*}
\mathbb{P}[|X|<\epsilon]\ge1-\sqrt{\frac{2}{\pi}}\frac{e^{-\epsilon^2/(2\sigma^2)}}{\epsilon/\sigma}
\end{align*}
\end{lemma}
\begin{lemma}[Laurent-Massart Inequality]
\label{lem:lm}
Letting $kX\sim\chi_k^2$, we have
\begin{align*}
\mathbb{P}\left[X-1<2\sqrt{\frac{\epsilon}{k}}+\frac{\epsilon}{k}\right]\ge1-e^{-\epsilon}
\qquad\text{and}\qquad
\mathbb{P}\left[X-1>-2\sqrt{\frac{\epsilon}{k}}\right]\ge1-e^{-\epsilon}.
\end{align*}
\end{lemma}
\begin{lemma}[Hoeffding's inequality]
\label{lem:hoeffding}
Given i.i.d. random variables $X_1,...,X_k$ where $X_i$ has mean $\mu_i$ and is $\sigma_i$-subgaussian, letting $\mu=k^{-1}\sum_{i=1}^k\mu_i$ and $\sigma^2=k^{-1}\sum_{i=1}^k\sigma_i^2$, then
\begin{align*}
\mathbb{P}\left[\left|\frac{1}{k}\sum_{i=1}^kX_i-\mu\right|<\sqrt{\frac{2\sigma^2\log(1/\delta)}{k}}\right]\ge1-\delta.
\end{align*}
\end{lemma}
\begin{lemma}[Binomial tail bound]
\label{lem:binom}
Suppose that $B\sim\text{Binomial}(k,p)$ is a Binomial random variable with $k$ samples and success probability $p$. Then, for any $b\ge kp$, we have
\begin{align*}
\mathbb{P}[B\le b]\ge1-e^{-2k(p-b/k)^2}.
\end{align*}
Alternatively, taking $b=kp+\sqrt{k\log(1/\delta)/2}$ yields
\begin{align*}
\mathbb{P}\left[B\le kp+\sqrt{\frac{k\log(1/\delta)}{2}}\right]\ge1-\delta.
\end{align*}
\end{lemma}
\begin{proof}
This result follows from Lemma~\ref{lem:hoeffding}.
\end{proof}

\end{document}